\documentclass{article}

\usepackage{PRIMEarxiv}

\usepackage[utf8]{inputenc} % allow utf-8 input
\usepackage[T1]{fontenc}    % use 8-bit T1 fonts
\usepackage{hyperref}       % hyperlinks
\usepackage{url}            % simple URL typesetting
\usepackage{booktabs}       % professional-quality tables
\usepackage{amsfonts}       % blackboard math symbols
\usepackage{nicefrac}       % compact symbols for 1/2, etc.
\usepackage{microtype}      % microtypography
\usepackage{lipsum}
\usepackage{fancyhdr}       % header
\usepackage{graphicx}       % graphics
\graphicspath{{media/}}     % organize your images and other figures under media/ folder

\usepackage{algorithm}
\usepackage{algorithmic}
\usepackage{mathtools}
\usepackage{amsthm}
\usepackage{subfigure}
\usepackage{caption}
\usepackage{color}
\usepackage{graphicx}
\usepackage{pifont}

\newtheorem*{theorem*}{Theorem}
\newtheorem{theorem}{Theorem}
\newtheorem{lemma}{Lemma}
\newtheorem*{corollary*}{Corollary}
\newtheorem{corollary}{Corollary}
\newtheorem{definition}{Definition}
\newtheorem*{definition*}{Definition}

\newtheorem{proposition}{Proposition}
\newtheorem*{proposition*}{Proposition}

\def\cS{{\mathcal{S}}}
\def\cD{{\mathcal{D}}}
\def\cA{{\mathcal{A}}}
\def\cN{{\mathcal{N}}}
\def\cM{{\mathcal{M}}}

\def\cT{{\mathcal{T}}}
\def\cL{{\mathcal{L}}}
\def\cR{{\mathcal{R}}}
\def\hT{{\widehat{\mathcal{T}}}}

\def\EE{{\mathbb{E}}}
\def\RR{{\mathbb{R}}}

\def\hD{{\widehat{\mathcal{D}}}}

\usepackage{color}

\usepackage{booktabs}
\usepackage{proof}
\usepackage{amsmath,amsfonts,amsthm,bm}
\usepackage{newfloat}
\usepackage{nicefrac}
\usepackage{listings}
\usepackage{wrapfig}
\usepackage{makecell}
\usepackage{bbm}
\usepackage{multirow}
\usepackage{wrapfig}
\usepackage{enumitem}
\usepackage{graphbox}
\usepackage{ragged2e}
\title{Diverse Randomized Value Functions: A Provably Pessimistic Approach for Offline Reinforcement Learning
}

\author{
  Xudong Yu \\
  Harbin Institute of Technology \\
  \texttt{hit20byu@gmail.com} \\
  \And
  Chenjia Bai \\
  Shanghai AI Laboratory \\
  \texttt{baichenjia@pjlab.org.cn} \\
  \And
  Hongyi Guo \\
  Northwestern University \\
  \texttt{Hongyiguo2025@u.northwestern.edu} \\
  \And
  Changhong Wang \\
  Harbin Institute of Technology \\
  \texttt{cwang@hit.edu.cn} \\
  \And
  Zhen Wang \\
  Northwestern Polytechnical University \\
  \texttt{zhenwang0@gmail.com} \\
}

\begin{document}
\maketitle

\begin{abstract}
Offline Reinforcement Learning (RL) faces distributional shift and unreliable value estimation, especially for out-of-distribution (OOD) actions. To address this, existing uncertainty-based methods penalize the value function with uncertainty quantification and demand numerous ensemble networks, posing computational challenges and suboptimal outcomes. 
In this paper, we introduce a novel strategy employing diverse randomized value functions to estimate the posterior distribution of $Q$-values. It provides robust uncertainty quantification and estimates lower confidence bounds (LCB) of $Q$-values. By applying moderate value penalties for OOD actions,  our method fosters a provably pessimistic approach. We also emphasize on diversity within randomized value functions and enhance efficiency by introducing a diversity regularization method, reducing the requisite number of networks. These modules lead to reliable value estimation and efficient policy learning from offline data.
Theoretical analysis shows that our method recovers the provably efficient LCB-penalty under linear MDP assumptions. Extensive empirical results also demonstrate that our proposed method significantly outperforms baseline methods in terms of performance and parametric efficiency.
\end{abstract}

% keywords can be removed

\section{Introduction}
Offline Reinforcement Learning, which aims to learn policies from a fixed dataset without interacting with the environment \cite{levine2020offline}, is appealing for applications that are costly to perform online exploration \cite{robot-2018, mt-opt}. Compared to online RL, offline RL provides the advantage of avoiding potentially risky interactions, making it suitable for scenarios such as autonomous driving, healthcare, and industrial applications. Additionally, as a \emph{data-driven} paradigm, offline RL can utilize large-scare datasets and recent advances in fields that adopt large-scale training data \cite{dt, tt}. A key challenge in offline RL is the distributional shift, which arises from the discrepancy between the learned policy and the behavior policy. This can lead to extrapolation errors when estimating the Q-values for OOD data, further exacerbated during Bellman updates due to bootstrapping and function approximation. As a result, the distributional shift problem can significantly degrade the performance of agents.

Previous offline RL methods have attempted to address the distributional shift problem by either restricting the learned policy to stay within the distribution or support of the dataset \cite{bcq-2019, bear-2019} or penalizing OOD actions to perform value regularization \cite{cql-2020}. In these approaches, agents are expected to choose in-sample or in-distribution actions as much as possible and avoid any OOD actions. However, these methods rely heavily on the accurate estimation of behavior policies or prevent the agent from selecting any OOD actions, without distinguishing between potentially good or bad OOD actions. Accordingly, these approaches often result in overly conservative policies. In contrast, we posit that there may be potentially good OOD actions, and it is more effective to identify them with an accurate uncertainty quantifier for their values.

Uncertainty-based methods provide an alternative and promising way to address the distributional shift by characterizing the uncertainty of $Q$-values and pessimistically updating the estimation. In this case, accurate and reliable uncertainty quantification is crucial, especially for OOD data points. To this end, PBRL \cite{PBRL-2022} measures the uncertainty by the disagreement of bootstrapped $Q$-functions and performs OOD sampling to handle the extrapolation error. Similarly, SAC-N \cite{EDAC-2021} enhances the critic with a large number of ensembles and achieves strong performance in offline environments. In spite of such progress, these methods require lots of ensembles, which can be computationally expensive and impractical in many cases. Furthermore, a large number of ensembles may suffer from ensemble performance saturation and mode collapse to several modes \cite{d2021repulsive}, indicating that ensemble members tend to produce similar predictions. They also lack the consideration of a prior distribution and fail to ensure the diversity among predictions to characterize the true posterior. Therefore, the following question arises: \emph{\textbf{Can we obtain a reliable uncertainty quantification in a principled way with fewer ensembles?}}

In this paper, we propose Diverse Randomized Value Functions (DRVF) for offline RL, which combine randomly initialized $Q$ ensembles and diversity constraints and enforce pessimism with a lightweight approximate posterior. Our approach incorporates two components for principled and reliable uncertainty estimation: approximate Bayesian learning with randomized $Q$ ensembles and repulsive regularization. Firstly, the posterior distribution of the value function can be estimated from randomzied value functions, and then the low-confidence-bound (LCB) of the $Q$-posterior contributes to pessimistic value updates. Theoretically, the Bayesian uncertainty defined by the disagreement among different $Q$-samples from the posterior recovers the provably efficient LCB-penalty in the Bayesian linear regression case. In practice, exact Bayesian inference is computationally infeasible for large-scale problems. Therefore, we use a Bayesian neural network (BNN) in the last layer of the $Q$-network to learn an approximate Bayesian posterior. We show that by combining BNNs and ensembles, DRVF can provide reliable uncertainty quantification and perform more efficient value updates. 

In addition to approximate Bayesian learning, we propose a repulsive regularization term to promote diversity among different samples from the ensemble BNNs. Maintaining diversity among these samples is crucial for both achieving high performance and ensuring computational efficiency. The repulsive term restrains different $Q$ samples from collapsing into the same function by regularizing the OOD data sampled following the learned policy. This regularization improves uncertainty quantification by enforcing diversity among posterior samples. We also show that the repulsive force helps to recover the true Bayesian posterior by building theoretical connections with the gradient flow for the optimization problem. DRVF obtains the uncertainty estimation by sampling several $Q$-values from the posterior distribution and learns the policy by following the LCB of the value function.
Compared to SAC-N that heavily depends on the number of networks in the ensemble, DRVF quantifies the uncertainty with a small number of ensembles by sampling multiple $Q$ values from the ensemble BNNs.
Empirical evaluations on the D4RL benchmarks \cite{d4rl-2020} demonstrate the favorable performance of our proposed approach, showcasing improved parametric efficiency with fewer ensembles compared to previous uncertainty-based methods. Our contributions can be summarized as follows:
\begin{itemize}
    \item We propose DRVF, a lightweight uncertainty-based algorithm for offline RL. DRVF employs BNNs to approximate the Bayesian posterior by leveraging a small number of ensembles and measures uncertainty via posterior sampling.
    \item We introduce a repulsive regularization term to enhance the diversity among the samples from the ensemble BNNs, leading to improved parametric efficiency.
    \item Theoretically, we establish the relationship between the proposed Bayesian uncertainty and the LCB penalty, providing provable efficient pessimism in linear MDPs. We also demonstrate that the repulsive term aids in recovering the true Bayesian posterior.
    \item We perform multiple experiments and ablation studies to evaluate our algorithm. The empirical results illustrate competitive performance and reasonable uncertainty quantification.
\end{itemize}

The rest of this paper is structured as follows. Section 2 first provides preliminary backgrounds of linear MDPs and offline RL, and Section 3 describes the related work. In the next section, we introduce the Bayesian uncertainty based on randomized value functions, including the approximation of the Bayesian posterior (Section 4.1), the LCB penalty (Section 4.2), and the pessimistic learning procedure (Section 4.3). Section 5 presents the repulsive regularization term, including our intuition and concrete implementations. Section 6 presents the results of our empirical study, including performance, uncertainty quantification, computational efficiency, and ablation studies, and Section 7 concludes the paper. Proofs and more experimental details can be found in the Appendix.

% \newdefinition{Assumption}{Assumption}

% \newdefinition{Remark}{Remark}
% \newtheorem{Theorem}{Theorem}
% % \newdefinition{Definition }{Definition }
% \newdefinition{Definition}{Definition}

\section{Preliminaries}
In this section, we introduce preliminary background about episodic Markov Decision Process (MDP), linear MDP, and Offline reinforcement learning.

\subsection{Episodic MDP}
We consider an Markov Decision Process with finite horizons, formulated as $\cM=(\cS,\cA,P, r,\gamma,T)$. It contains the state space $\cS$, the action space $\cA$, the transition dynamics $P$, and the reward function $r:\cS\times\cA\rightarrow\RR$, as well as the discount factor $\gamma\in(0,1]$ and the episode length $T$. 
In RL, a policy $\pi(a|s)$ is learned to maximize the cumulative discounted reward as $\EE[\sum_{t=0}^{T}\gamma^t r(s_t,a_t)]$. This objective can be optimized by estimating a state or state-action value function, which gives the expected cumulative reward and then recovers a near-optimal policy. We present the recursive definitions for these value functions as 
\begin{equation}
    V(s)=\mathbb{E}_{a\sim \pi(\cdot|s)}[Q(s,a)],
    \quad Q(s,a)=r(s,a)+\gamma \mathbb{E}_{s'\sim P(\cdot|s,a)}[V(s')],
\end{equation} 
where the superscript $\pi$ is omitted.
These two equations can be combined to express the $Q(s,a)$ in terms of $Q(s',a')$:
% \cite{levine2020offline}: 
\begin{equation}
Q(s,a)=r(s,a)+\gamma \mathbb{E}_{s'\sim P(\cdot|s,a),a'\sim \pi(\cdot|s)}[Q(s',a')].
\end{equation}

In this work, we use a parameterized state-action value function $Q_w(s,a)$, where $w$ denotes the parameter of the Q network. The $Q$-function is learned corresponding to the current policy $\pi(\cdot|s)$ and obeys the following equation: 
\begin{equation}
    Q_w(s,a)=r(s,a)+\gamma \mathbb{E}_{s'\sim P(\cdot|s,a),a'\sim \pi(\cdot|s)}[Q_w(s',a')].
\end{equation}
This equation can be expressed as the Bellman operator, which indicates the learning target for the $Q$-network. By employing the target network that is softly updated for stability \cite{DQN-2015}, the Bellman operator is given by:
\begin{equation}
\label{eq:offline}
    \hT Q_w(s,a)=r(s,a)+\gamma \EE_{s'\sim P(\cdot|s,a),a'\sim \pi(\cdot|s)}[Q_{w^{-}}(s',a')],
\end{equation}
where $w^{-}$ represents the parameter of the target $Q$ network. The $Q$-function is learned by minimizing the Bellman residual error as $\EE_{s'\sim P(\cdot|s,a), a\sim \pi(\cdot|s)} [Q(s,a)-\cT Q(s,a)]^2$.

In offline RL, the agent needs to learn a policy from a static dataset $\cD_m=\{s_t^k,a_t^k,r_t^k,s_{t+1}^k\}_{k\in[m]}$ that contains $m$ episodes. The Bellman operator $\hT Q$ is then computed by estimating the expectation based on the offline datasets $\cD_m$, and the Bellman error becomes $\EE_{(s,a,r,s')\sim\cD_m}[Q(s,a)-\hT Q(s,a)]^2$ by sampling the experiences from the dataset. The policy $\pi(\cdot|s)$ is learned by maximizing the value function $Q(s,a)$. However, the problem of offline RL is that the policy $\pi(\cdot|s)$ can greedily choose OOD actions, which are not covered in $\cD_m$. These OOD actions cause high extrapolation errors in the value function estimation, known as distributional shift. Since these errors cannot be corrected by online interactions, they often result in sub-optimal policies.

\subsection{Linear MDP}
Our theoretical derivations have close connections with linear MDPs and Least Squares Value Iteration (LSVI). The former is a widely used assumption \cite{pevi-2021, lsvi-2020}, while the latter is a classic method frequently used in the linear MDPs to compute the closed-form solution.
In the following, we introduce the pessimistic value estimation in linear MDPs \cite{lsvi-2020,pevi-2021}, where the transition dynamics and reward function are expressed for $\forall(s_{t}, a_t, s_{t+1})\in\mathcal{S}\times\mathcal{A}\times\mathcal{S}$, as
\begin{equation}
\label{eq::pf_def_linearMDP}
P(s_{t+1} \,|\, s_t, a_t) = \langle \psi(s_t, a_t), \varphi(s_{t+1}) \rangle, \quad
r(s_t, a_t) = \psi(s_t, a_t)^\top \upsilon,
\end{equation}
where $\psi: \cS\times\cA\mapsto \RR^d$ represents a known feature embedding. Meanwhile, the reward function $r:\cS\times\cA\mapsto[0, 1]$ and the feature $\|\psi\|_2 \leq 1$ are also assumed to be bounded.
% We consider the settings of $\gamma=1$ in the following.

Based on the above assumption, the state-action value function can also be linear to $\psi$, as
\begin{equation}
\begin{aligned}
Q^{\pi}(s_t,a_t)
&=r(s_t,a_t)+\mathbb{E}_{s' \sim P(\cdot|s,a)}V(s') \\
&=\psi(s_t, a_t)^\top \upsilon + \int_{s'\in \mathcal{S}}V^{\pi}(s')\langle \psi(s_t, a_t), {\rm d}\varphi(s_{t+1}) \rangle \\
&=\psi(s_t, a_t)^{\top}w,
\end{aligned}
\end{equation}
where $w=\upsilon + \int_{s'\in \mathcal{S}}V^{\pi}(s') {\rm d}\varphi(s_{t+1})$.
In offline RL settings with a fixed dataset $\cD_m=\{s_t^k,a_t^k,r_t^k,s_{t+1}^k\}_{k\in[m]}$, the parameter of the $w$ can be solved via least-squares value iteration algorithm, as
\begin{equation}
\label{eq::appendix_OOD_LSVI}
\widehat w_t = \min_{w\in \RR^d} \sum^m_{k = 1}\bigl(\psi(s^k_t, a^k_t)^\top w - r(s^k_t, a^k_t)- V_{t+1}(s^k_{t+1})\bigr)^2. % + \lambda\cdot \|w\|_2^2.
\end{equation}
% where $V_{t+1}$ is the estimated value function in the $(t+1)$-th step. 
Then we have an explicit solution to Eq. (\ref{eq::appendix_OOD_LSVI}) written by
\begin{equation}
\label{eq::pf_tilde_w}
\widehat w_t = \Lambda^{-1}_t\sum^m_{k = 1}\psi(s^k_t, a^k_t)y_t^k,\quad
\text{where } \Lambda_t = \sum^m_{k=1}\psi(s^k_t, a^k_t)\psi(s^k_t, a^k_t)^\top
\end{equation}
is the feature covariance matrix of the state-action pairs in the offline dataset, and $y_t^k=r(s^k_t, a^k_t)+ V_{t+1}(s^k_{t+1})$ indicates the Bellman target in regression.

\section{Related Work}

\subsection{Offline RL}

One of the key challenges in offline RL is the distributional shift problem, which arises when the learned policy deviates from the behavior policy. Most offline RL research aims to address it and focuses on avoiding overestimation of the $Q$-values for OOD actions. Specifically, previous works
(\romannumeral 1) restrict the learned policy to the static dataset by explicitly estimating the behavior policy \cite{bcq-2019, brac-2019} or implicitly regularizing policy learning \cite{bear-2019, AWR, awac-2020, td3bc-2021}; (\romannumeral 2) utilize value regularization \cite{cql-2020} and pessimistically update the $Q$-values for OOD actions; (\romannumeral 3) perform in-sample learning from the dataset without querying the value of unseen actions \cite{iql, policy-guided}. These methods can mitigate the distributional shift by constraining the actions to be in-distribution or in-sample, or regularizing the value function. 
However, they often restrict the agent from taking any OOD actions and lack the consideration of whether such actions might be good or bad. As a result, these methods can be overly conservative and may ignore potentially good OOD actions. 

Alternatively, uncertainty-based algorithms like PRBL \cite{PBRL-2022} and SAC-N \cite{EDAC-2021} use an ensemble of $Q$ networks to perform pessimistic policy updates. In particular, our work is related to PBRL, since we both rely on the linear MDP assumption and obtain the lower confidence bound (LCB) penalty by quantifying the uncertainties of Q functions. PBRL utilizes the bootstrapped Q-functions to calculate the uncertainty and approximate a non-parametric Q-posterior, while DRVF provides theoretical guarantees to obtain the true Bayesian posterior. During the Bellman updates of Q functions, PBRL explicitly estimates the pseudo TD targets for OOD samples and requires a tuning parameter to stabilize the learning process. In contrast, DRVF does not modify the TD targets and applies a repulsive regularization term on OOD samples to diversify uncertainty estimation. Compared to PBRL and SAC-N, which still need 10$\sim$50 independent networks for uncertainty quantification, DRVF provide a novel and lightweight architecture, achieving improved parametric efficiency.
% It obtains strong performance but is computationally expensive. Other methods reduce computation by performing additional OOD sampling \cite{PBRL-2022} or diversifying the gradient of value functions \cite{EDAC-2021}. Despite this, they . In contrast, our method improves value estimation for OOD data by using lightweight and reliable uncertainty quantification. Our method only need a few ensembles.

There are also model-based offline approaches that aim to learn a pessimistic MDP \cite{mopo-2020, morel-2020}, use ensembles to capture the uncertainty of dynamics models \cite{mopo-2020, morel-2020, bremen}, combine conservative value estimation with dynamic models \cite{combo-2021}, or model the trajectory distribution using a high-capacity Transformer \cite{dt, tt, wang2022bootstrapped}. These methods leverage dynamics models to enhance sample efficiency and expand the data distribution through planning or rollout. Nevertheless, they rely on accurate dynamics models or require additional modules such as Transformers. On the contrary, our method follows a model-free manner, prioritizing computational efficiency. Recently, an adversarial trained method \cite{ATAC-2022} based on a two-player Stackelberg game has shown promising results, albeit at a higher computational cost.

\subsection{Uncertainty Quantification}
Uncertainty quantification is widely used for exploration in the online RL setting \cite{count,sunrise-2021, rvf}, i.e., \emph{optimism in the face of uncertainty}, while its offline counterpart tends to be \emph{pessimism in the face of uncertainty} \cite{pevi-2021}. Because of the limited coverage of offline datasets and the extrapolation error from OOD actions, it is more challenging to obtain accurate uncertainty quantification in an offline setting. 

One approach to estimate predictive uncertainty is to use ensembles of neural networks. This can be achieved through independent initializations \cite{rem, PBRL-2022, EDAC-2021}, different hyperparameters \cite{hyper-ensemble}, varying network depths \cite{depth-ensemble}, or early exits \cite{earlyexit}. However, large numbers of neural networks are required to approximate the posterior and may suffer from model collapse.
Another approach is to use Bayesian neural networks to learn the posterior over the parameters of the model. 
Since computing the exact posterior is often intractable, previous methods have resorted to approximations using Markov chain Monte Carlo \cite{welling2011bayesian}, variational inference \cite{blundell2015weight}, or Monte-Carlo dropout \cite{dropout-2016, hiraoka2022dropout, UWAC-2020}.
Noisy-Net \cite{noisy-net} injects noise to the parameters but is not ensured to approximate the posterior distribution of the parameters. Our method is also related to Hyper-DQN \cite{hyperdqn}, which uses Thompson sampling and hyper-network for exploration, while we use variational inference to estimate the posterior distribution for pessimism.

One major challenge associated with ensemble-based methods is the need to ensure diversity among ensemble members, as the quality and performance of the ensemble are heavily affected by it \cite{d2021repulsive, hiraoka2022dropout}. To address this problem, several approaches have been proposed to improve the diversity of the ensemble without sacrificing individual accuracy. For example, the Random Prior \cite{prior-2018} employs independent prior networks for ensemble members, Hyper-Ensemble \cite{wenzel2020hyperparameter} introduces different hyper-parameters for ensemble members, and Stein variational gradient descent (SVGD) \cite{liu2016stein} incorporates a kernelized repulsive force in the parameter space. Inspired by these methods, we propose a simple repulsive term to regularize the Bayesian posterior. 

\section{Bayesian Uncertainty from Randomized Value Functions}
In this section, we describe how to obtain uncertainty quantification of randomized value functions using a combination of ensembles and BNNs. Firstly, we approximate the Bayesian posterior using ensemble BNNs. Secondly, under the assumption of linear MDP, the uncertainty can be characterized by the standard deviation of the estimated posterior, which is related to the LCB penalty. Finally, we show how to perform pessimistic Bellman updates of the value function based on the estimated uncertainty, which leads to better estimation of the $Q$-values for OOD data.

\subsection{Approximation of the Q-posterior}

In this part, We provide the motivation and foundation of the proposed uncertainty quantification method in linear MDPs \cite{pevi-2021}. As mentioned before, the reward function and transition dynamics are linear in the state-action feature $\psi(s,a)$. Considering that we have a learned representation $\psi(s,a)$, the value function $Q(s,a)=\psi(s,a)^\top w$ can also be linear in $\psi(s,a)$,
where $w$ is the underlying true parameter of the linear $Q$-function. More details on linear MDPs can be found in related work \cite{pevi-2021, lsvi-2020}.

In DRVF, we approximate the posterior of the value function $\mathbb P(\tilde Q \,|\, s, a, \cD_m)$ via non-linear neural networks. Specifically, we consider $\psi(s, a)$ as the output of \emph{hidden layers} of the $Q$-network, and $w$ as the weight of the \emph{last layer}. This allows us to express the value function as $Q(s,a)=\psi(s,a)^{\top} w$.
% +\epsilon$, where the noise term $\epsilon$ captures the aleatoric uncertainty in the $Q$-value. 
If we denote the variational distribution of the posterior of $w$ as $q_\theta(w \,|\, \mathcal D_m)$, where $\theta$ represents the parameters, then the true posterior $\mathbb P(w \,|\, \mathcal D_m)$ can be approximated using a variational inference approach \cite{blundell2015weight}. Note that
\begin{equation}
\label{eq:elbo}
\mathbb P(\cD_m)
= \EE_{\mathbb P(w \,|\, \cD_m)} [\mathbb P(\cD_m \,|\, w)]
= \underbrace{\EE_{q_\theta(w \,|\, \cD_m)} [\log \mathbb P(\cD_m \,|\, w)] - \mathrm{KL}(q_\theta (w \,|\, \cD_m) \,\|\, \mathbb P(w))}_\text{ELBO} + \mathrm{KL}(q_\theta(w \,|\, \cD_m) \,\|\, \mathbb P(w \,|\, \cD_m)).
\end{equation}
We denote by $q^*_\theta$ the optimal solution to maximizing the evidence lower bound (ELBO). According to Eq.~\eqref{eq:elbo}, maximizing the ELBO term is equivalent to minimizing the Kullback-Leibler (KL) divergence between the true posterior $\mathbb P(w \,|\, \cD_m)$ and our estimation $q(w \,|\, \cD_m)$, which leads to the desired result $q^*(w \,|\, \cD_m) = \mathbb P(w \,|\, \cD_m)$. We denote by $\tilde Q_w$ the random variable for the estimated state-action value with a corresponding weight $w \sim q^*(w \,|\, \cD_m)$.

\subsubsection{The ELBO learning objective}

To derive the practical learning objective of DRVF based on Eq.~\eqref{eq:elbo}, we construct a new offline dataset $\widehat \cD_m=\{\psi_t^i,y_t^i\}_{i\in[m]}$, where the elements can be obtained based on dataset $\cD_m$ by calculating the state-action feature $\psi(s_t^i,a_t^i)$ and the target-value function $y_t^i=\hT Q(s_t^i,a_t^i)$ by Eq.~\eqref{eq:offline}. In our analysis, we consider $y_t^i$ as a fixed learning target based on the current $Q$-function. As a result, the optimal solution of $q^*_\theta$ in the ELBO becomes
\begin{equation}\label{eq:approx-elbo}
\begin{aligned}
q^*_\theta &= \arg\max \EE_{q_\theta(w \,|\, \widehat \cD_m)} \big[\log \mathbb P\big(\widehat \cD_m \,|\, w\big)\big] - \mathrm{KL}\big(q_\theta (w \,|\, \widehat \cD_m) \,\|\, \mathbb P(w)\big)\\
&= \arg\max \EE_{q_\theta(w \,|\, \widehat \cD_m)}\left[\sum\nolimits_{i\in[m],t\in[T]} \log \mathbb P\big(y_t^i \,|\, \psi_t^i, w\big)\right] - \mathrm{KL}\big(q_\theta (w \,|\, \hD_m) \,\|\, \mathbb P(w)\big)\\
&\approx \arg\max \sum\nolimits_{k\in[n]} \sum\nolimits_{i\in[m],t\in[T]} \log \mathbb P\big(y_t^i \,|\, \psi_t^i, w^{(k)}\big) -  \mathrm{KL}\big(q_\theta (w \,|\, \hD_m) \,\|\, \mathbb P(w)\big),
\end{aligned}
\end{equation}
where the last step uses Monte Carlo sampling to approximate the variational posterior, and $w^{(k)}$ denotes the $k$-th sample drawn from $q_\theta(w \,|\, \widehat \cD_m)$. To optimize the objective using gradient backpropagation, we suppose that the posterior is a diagonal Gaussian distribution and apply the re-parameterization trick to sample several random vectors $\{z^{(k)}\}_{k\in[n]}$ from a standard Gaussian. The posterior weight is then calculated as $\{w^{(k)}=\sigma_w^\top z^{(k)}+\mu_w\}_{k\in[n]}$, where $q_\theta(w)=\cN(\mu_w,\Sigma_w)$ and $\Sigma_w = \sigma_w \sigma^\top_w$. The $Q$-value is obtained by $\{Q^{(k)}=\psi^{\top} w^{(k)}\}_{k\in[n]}$.
\begin{itemize}
    \item For the first term in Eq.~\eqref{eq:approx-elbo}, we simplify the target by using a deterministic response, then the log-likelihood reduces to Bellman residual error between the target $\hT Q^{(k)}$ and $Q^{(k)}$. This error is calculated by taking the average over state-action pairs and posterior samples.
    \item For the second term in Eq.~\eqref{eq:approx-elbo}, the prior function also follows a standard Gaussian distribution. And the KL term is easy to optimize.
\end{itemize}

In practice, we use $\psi(s,a)$ as part of the $Q$-network to extract the feature of state-action pairs. We simultaneously learn the representation $\psi(s,a)$ and the posterior distribution by following the gradients of Eq.~\eqref{eq:approx-elbo}.

\subsection{LCB penalty from the ensemble BNNs}

Given the approximated $Q$-posterior, we can quantify the uncertainty of $Q$-values. In the context of linear MDP, the standard deviation of the estimated posterior distribution is a LCB term and $\xi$-uncertainty quantifier, which guarantees the pessimism for the value function.

\begin{theorem}
\label{thm:uncertainty} 
Under linear MDP assumptions, it holds for the standard deviation of the estimated posterior distribution $\mathbb P(\tilde Q \,|\, s, a, \cD_m)$ that
% follows
\begin{equation}
\label{eq:lcb}
\mathrm{Std}(\tilde Q \,|\, s_t, a_t, \cD_m) = \left[\psi(s_t, a_t)^\top \Lambda^{-1}_t \psi(s_t, a_t)\right]^{\nicefrac{1}{2}}
:= \Gamma_t^{\rm lcb}(s_t,a_t),
\end{equation}
where $\Lambda_t = \sum\nolimits_{i \in [m]} \psi(s^i_t, a^i_t) \psi(s^i_t, a^i_t)^\top + \lambda \cdot \mathbf I$, and $\Gamma_t^{\rm lcb}(s_t,a_t)$ is defined as the LCB-term.
\end{theorem}
Theorem~\ref{thm:uncertainty} establishes the equivalence between the standard deviation of the $Q$-posterior and the LCB-term $\Gamma^{\rm lcb}_t(s_t,a_t)$ in linear MDPs.
The proof is given in Appendix A. The feature covariance matrix $\Lambda_t$ accumulates the state-action features during learning the value function $Q$. The LCB term defined in Eq.~\eqref{eq:lcb} measures the confidence interval of the $Q$-functions with the offline data $\cD_m$. 
In the following, We propose that $\Gamma_t^{\rm lcb}(s,a)$ serves as a $\xi$-uncertainty quantifier \cite{pevi-2021} and in favor of characterizing  the optimality gap of the value function.

\begin{definition}[\textbf{$\xi$-Uncertainty Quantifier \cite{pevi-2021}}]
For all $(s,a)\in{\mathcal{S}}\times\mathcal{A}$, if the inequality 
$|\widehat{\mathcal{T}}V_{t+1}(s,a)-\mathcal{T}V_{t+1}(s,a)|\leq\Gamma_{t}(s,a)$
holds with probability at least $1-\xi$, then $\{\Gamma_{t}\}_{t\in[T]}$ is a $\xi$-Uncertainty Quantifier. $\widehat{\mathcal{T}}$ empirically estimates $\mathcal{T}$ based on the data.
\label{def:uncertainty-quantifier}
\end{definition}

\begin{proposition}
    The LCB term $\Gamma_t^{\rm lcb}(s,a)$ is a $\xi$-uncertainty quantifier.
    \label{xi-uncertainty}
\end{proposition}

Proof for Proposition \ref{xi-uncertainty} is put in the Appendix A, and we refer to \cite{pevi-2021, PBRL-2022} for more details.

\begin{corollary} [\textbf{Suboptimality Gap \cite{pevi-2021}}]
Since $\Gamma_{t}^{\rm lcb}(s_{t},a_{t})$ is a $\xi$-uncertainty quantifier,
the suboptimality gap satisfies
\begin{equation}
    V^{*}(s_{1})-V^{\pi_{1}}(s_{1})\leq\sum^{T}_{t=1}\mathbb{E}_{\pi^{*}}\bigl{[}%
\Gamma_{t}^{\rm lcb}(s_{t},a_{t}){\,|\,}s_{1}\bigr{]}.
\end{equation}
\label{suboptimality}
\end{corollary}
This corollary implies that the pessimism of the value function holds with probability at least
$1-\xi$ as long as
$\Gamma_{t}^{\rm lcb}(s_{t},a_{t})$ 
are $\xi$-uncertainty quantifiers. We put the related proof in Appendix A. We remark that while the above analysis is based on the linear MDP assumption, our method can be readily applied in high-dimensional space and deep neural networks.

\subsection{Pessimistic Q-learning}

Based on the LCB term $\Gamma_t^{\rm lcb}(s,a)$, the pessimistic value function $Q_{\rm lcb}(s,a)=Q(s,a)-\Gamma_t^{\rm lcb}(s,a)$ can be applied in the Bellman updates and policy updates. During the Bellman updates, updating the value function with uncertainty penalty is known to be information-theoretically optimal in linear MDPs with finite horizon \cite{pevi-2021}. For policy updates, optimizing the LCB of the value function also restrict the policy from discarding potentially beneficial OOD actions. DRVF can determine whether an OOD action is `good' based on its LCB and select the better behavior. If the LCB of an OOD action is higher than that of in-distribution actions, it still has the potential to be selected.

There are two methods to obtain the pessimistic value function. 
\begin{enumerate}
    \item We can calculate the standard deviation of the posterior as $\Gamma_t(s_t,a_t)=\sqrt{\psi(s_t,a_t)^{\top} \Sigma_w \psi(s_t,a_t)}$ to serve as an estimation of $\Gamma_t^{\rm lcb}$ in Eq.~\eqref{eq:lcb}. The pessimistic value function is then obtained by applying the LCB as a penalty, i.e., $Q^1_{\rm lcb}(s_t,a_t)=\bar{Q}(s_t,a_t)-\alpha\Gamma_t(s_t,a_t)$, where $\bar{Q}(s_t,a_t)=\psi(s_t,a_t)^{\top} \mu_w$ is the mean $Q$-value.
    \item Alternatively, we can sample several $Q$-values from the posterior distribution, and obtain the pessimistic $Q$-value by taking the minimal value of the $Q$-samples as $Q^2_{\rm lcb}(s_t,a_t)=\min_{k\in[n]} Q^{(k)}(s_t,a_t)$.
\end{enumerate}

We remark that $Q^1_{\rm lcb}$ and $Q^2_{\rm lcb}$ are approximately equivalent with an appropriate $\alpha$. Since the realization $Q^{(k)}$ has a linear mapping to $w^{(k)}$, $Q^{(k)}$ also follows a Gaussian distribution. According to prior works \cite{EDAC-2021}, we have
\begin{equation}\label{eq:gaussian-lcb}
    \EE\big[\min\nolimits_{k\in[n]} Q^{(k)}(s,a)\big]
    \approx \bar{Q}(s,a) - \alpha \cdot {\rm Std}\big(\{Q^{(k)}(s,a)\}_{k\in[n]}\big) \approx \bar{Q}(s,a) - \alpha \Gamma_t(s,a),
\end{equation}
where $\alpha=\Phi^{-1}(\frac{N-\nicefrac{\pi}{8}}{N-\nicefrac{\pi}{4}+1})$, and $\Phi(\cdot)$ is the cumulative distribution function of the standard Gaussian. In practice, we use $Q^2_{\rm lcb}$ as the pessimistic value function since (\romannumeral 1) it is easy to combine with ensemble methods by sampling multiple values from both the ensemble and BNNs;
and (\romannumeral 2) the sampling method can be effectively calculated via matrix manipulation by using multiple noise vectors.

\begin{figure}[!tbp]
\centering
\includegraphics[width=\linewidth]{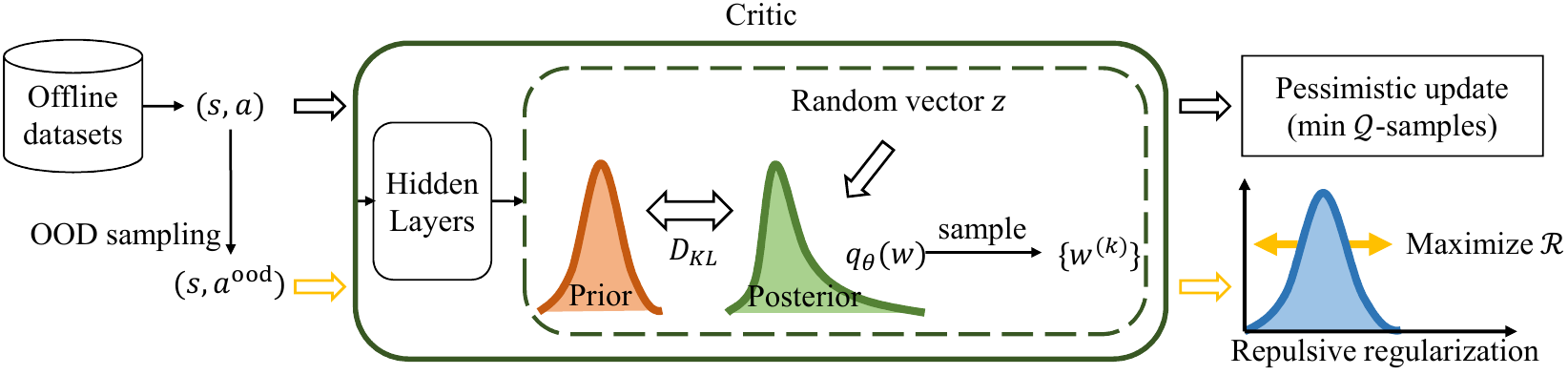}
\caption{The architecture of DRVF, where we utilize approximate Bayesian inference in the last layer of the critic network. We perform OOD sampling to obtain $(s,a^{\rm ood})$ pairs based on $(s,a)\sim \cD_m$. We input $(s,a)$ and $(s,a^{\rm ood})$ pairs to obtain the $Q$-samples and the repulsive term, respectively. Then $(s,a)$ is used for the pessimistic update of the $Q$-values, and $(s,a^{\rm ood})$ provides regularization.}
\label{fig:DRVF}
\end{figure}

The architecture of BNN in our critic network and the pessimistic update are presented in Fig.~\ref{fig:DRVF}.
We use $M$ ensemble BNNs as critic networks; then the pessimistic learning target for the $j$-th $Q$-network is given by:
\begin{equation}\label{eq:parc-lcb}
    \hT^{\rm lcb} Q_{w_j}(s,a) 
    = r(s,a) + \gamma\hat{\EE}_{\substack{s'\sim P(\cdot|s,a), a'\sim \pi_\psi(\cdot|s') 
    w_j^{(k-)}\sim q_\theta(w_j^-|\mathcal{D}_m)}}
    \Big[ \min_{\substack{{j\in[M]}\\ k\in[n]}}Q^{(k)}_{w_j^-}(s',a')-\beta \log \pi_\psi (a'|s')\Big],
\end{equation}
where $w_j$ and $w^{-}_j$ denote the Bayesian parameters for $Q$-networks and target networks and $\pi_\psi$ represents the policy network. In Eq.~\eqref{eq:parc-lcb}, we sample $n$ $Q$-values from each BNN and take the minimum value among all the $n\times M$ sampled $Q$-values of the $M$ BNNs to compute the target. 

\section{Repulsive Regularization for Diversity}
The estimated Bayesian posterior allows us to sample infinite $Q$-values from the posterior for a single $(s,a)$ pair. This sampling process is approximately equivalent to generating infinite $Q$-ensembles \cite{pearce2020uncertainty}. However, adding more ensemble members may lead to performance saturation, where ensembles tend to output similar estimates, and affect the quality of the uncertainty quantification \cite{d2021repulsive}. To maintain the diversity among the samples from ensemble BNNs, we propose a simple repulsive term for the Bayesian posterior. In this section, we first introduce the intuition for the repulsive term by considering its effect on the uncertainty quantification for OOD data. We then present the specific regularization term and ways to generate OOD data. Finally, we summarize the algorithm. 

\begin{figure}[ht]
    \centering
    \subfigure[]{\includegraphics[width=0.49\linewidth,align=t]{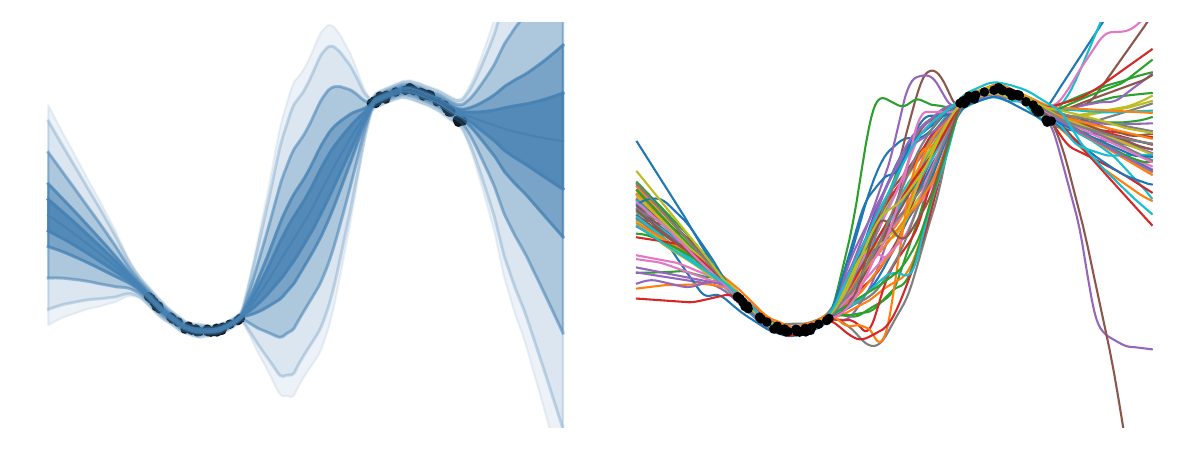}}
    \subfigure[]{\includegraphics[width=0.49\linewidth,align=t]{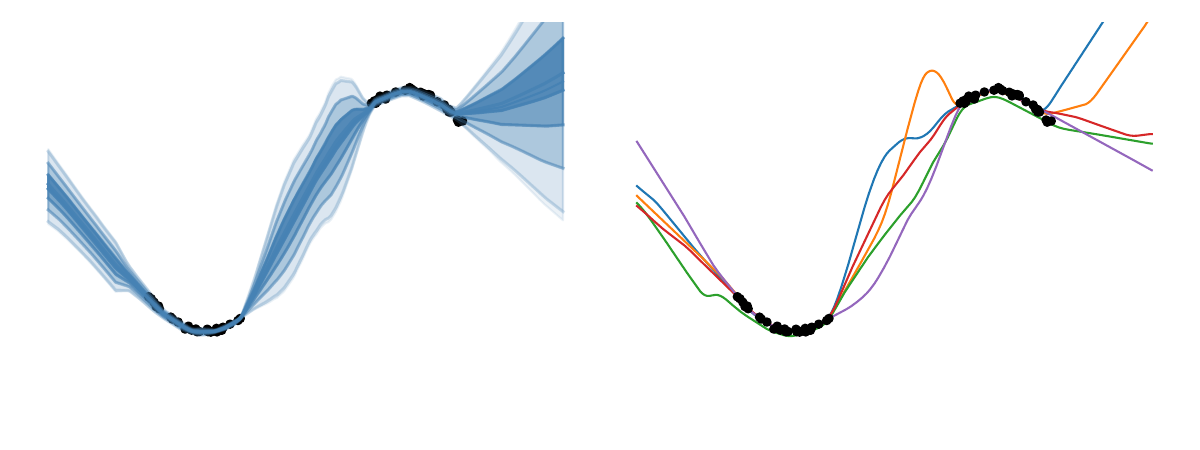}}
    \caption{Visualization for the intuition behind the repulsive regularization term. Blank dots represent data points, while colorful lines denote possible estimates of the Q-function. The shaded area denotes the uncertainty of ensemble predictions. In panel (a), uncertainty quantification and estimates are obtained from a large number of ensembles, which provides reliable uncertainty estimation. In panel (b), the uncertainty and possible estimates come from a small number of ensembles, which fails to properly account for the uncertainty. In DRVF, we explicitly maximize the standard deviation of the samples from the ensemble BNNs to obtain reliable uncertainty quantification with fewer ensembles.}
    \label{fig:max-ood}
\end{figure}

\subsection{Intuition Behind the Repulsive Term}

Our proposed repulsive regularization term relies on the idea that \emph{the uncertainty should be lower for in-distribution data and higher for OOD data.} We visualize this intuition by simplifying the Q-function as a one-dimensional function with respect to actions. As illustrated in Fig. \ref{fig:max-ood}, we expect smaller differences between ensemble predictions for in-distribution data and larger differences for OOD data to capture possible estimates for OOD data \cite{agree-to-disagree}. Moreover, the comparison between Figs. \ref{fig:max-ood} (a) and (b) highlights that more ensembles offer a better characterization of the uncertainty. 

Fig. \ref{fig:max-ood} also implies an intuitive depiction of the repulsive regularization term, designed to enhance the diversity of value estimations. We speculate that accurate characterization of the confidence bound of value estimations can potentially allow us to model the entire distribution of Q values for OOD data using a small number of ensembles. This can be achieved by encouraging ensemble members to be distinct from one another, thereby enabling us to estimate the confidence bound of value estimations even with a limited number of ensembles.  As a result, the LCB of the value function can be effectively characterized, which subsequently plays a crucial role in policy optimization.

To quantify the uncertainty with a reduced number of ensembles, we explicitly maximize the differences of ensemble predictions of Q-values for OOD actions, for example, the standard deviations of Q-values from the ensemble BNNs. We suggest that such regularization can encourage the ensemble Q-functions to provide different estimates for OOD data. Consequently, fewer ensembles are required to generate a reasonable uncertainty quantification, leading to improved computational efficiency.

\subsection{Repulsive term for ensemble BNNs}

In DRVF, we explicitly increase the diversity of different $Q$-samples to prevent the posterior from collapsing into the same solution. The repulsive regularization term $\cR(\cdot)$ can be applied in either the weight space (i.e., $\cR\big(\{w^{(k)}\}\big)$) or the function space (i.e., $\cR\big(\{Q^{(k)}\}\big)$). However, since the weight space has much higher dimensions, calculating the regularization in the function space is more efficient.
We define the repulsive term as
\begin{equation}\label{eq:repulsive}
\cR\big(Q(s, a^{\rm ood})\big) = {\rm Std}\big(\{Q^{(k)}(s, a^{\rm ood})\}_{k\in[n]}\big),
\end{equation}
where the OOD action $a^{\rm ood}$ is sampled from the learned policy $\pi(\cdot|s)$, $s$ is sampled from the dataset, and $Q^{(k)}$ is sampled from the posterior. By maximizing $\EE_{s,a^{\rm ood}\sim \pi}\big[\cR\big(Q(s, a^{\rm ood})\big)\big]$, we explicitly maximize the disagreement among $Q$-samples for OOD actions, which enhances the reliability of the uncertainty quantification. We remark that we do not enforce regularization for $(s,a)$ pairs sampled from the offline dataset, as the uncertainty tends to be small for in-distribution samples. The repulsive regularization is illustrated in Fig.~\ref{fig:DRVF}, where the offline data and OOD data are used for pessimistic Bellman updates and repulsive regularization, respectively.

In DRVF, we combine the proposed Bayesian posterior and the ensemble in a unified architecture using $M$ BNNs as an ensemble $\{Q^{(k)}_{j}\}_{k\in[n], j\in[M]}$. The repulsive regularization is calculated by sampling from both the ensemble (i.e., $j\in[M]$) and the Bayesian posterior (i.e., $k\in[n]$). As a result, DRVF combines ensemble, BNN, and repulsion to obtain a posterior. Compared to the previous uncertainty-based methods \cite{EDAC-2021,PBRL-2022}, our method is more efficient as it requires much fewer ensemble members (i.e., 2$\sim$5).

\subsection{Ways to Obtain OOD Actions}

We further discuss the way to generate OOD actions. 
In general, we can roughly regard any actions that are far from behavior policies as OOD actions. Since the behavior policy is unknown, we need to generate OOD actions by sampling from some distributions that are different from the behavior policy. Two ways to generate OOD actions can be considered: randomly sampling from the action space and sampling from the learned policy. 

As mentioned before, we obtain OOD actions by sampling from the learned policy, which is consistent with the previous work PBRL.
In Eq. \ref{eq:offline}, the value of the state-action pair $Q(s,a)$ is updated by regressing the target $\mathbb{E}[r+\gamma Q(s',a')]$. However, $Q(s',a')$ is often inaccurate in the offline setting since $a'$ can be OOD actions chosen by the learned policy. We suggest that sampling OOD actions from the learned policy leads to better pessimism for the neighbor field of the learned policy. On the other hand, randomly sampling from the action space is insufficient to enforce pessimism. Detailed comparisons, including empirical results, are provided in Section E.2.

\subsection{Learning with Bayesian Uncertainty}

We now introduce the practical learning algorithm of DRVF. Given the prior function and repulsive term, the loss function for training the $j$-th $Q$-network is
\begin{equation}\label{eq:critic}
\begin{aligned}
\cL_{\rm critic}^j \,=\, 
\eta_q \cdot \EE_{s,a,r,s'\sim \cD_m}\Big[\Big(Q_{w_j}(s,a)-\hT^{\rm lcb} Q_{w_j}(s,a)\Big)^2 + \mathrm{KL}\big(q_\theta (w_j) \,\|\, \mathbb P(w)\big)\Big] \\
-\eta_{\rm ood} \cdot \EE_{s\sim \cD_m,a^{\rm ood}\sim \pi}\,[\,\cR\big(Q_{w_j}(s, a^{\rm ood})\big)\,]\,,
\end{aligned}
\end{equation}
where $\hT^{\rm lcb} Q_{w_j}$ is defined in Eq.~\eqref{eq:parc-lcb}, $\eta_q$ and $\eta_{\rm ood}$ are constant factors, and $\cR(\cdot)$ is the repulsive term defined in Eq.~\eqref{eq:repulsive}. The first two terms in $\cL_{\rm critic}^j$ recover the variational ELBO defined in Eq.~\eqref{eq:approx-elbo}, where the first term maximizes the log-likelihood of the target, and the second term regularizes the posterior distribution. The last repulsive term maximizes the diversity of samples for OOD actions. The loss function is used to train the parameter $\theta$ of the
Gaussian distribution $q_\theta(w_j)$ and the representation $\psi(s,a)$.

Based on the value function, the policy network is updated by maximizing the minimal sampled value from the Bayesian posterior and ensemble, as
\begin{equation}
\label{eq:Lp}
\begin{aligned}
    \mathcal{L}_{\rm actor} 
    &= - \EE_{s\sim\cD_m, a\sim\pi_\psi(\cdot|s),
    w_j^{(k)}\sim q_\theta(w_j|\cD_m)}
\Big[\min\limits_{k\in[n],j\in[M]}Q^{(k)}_{w_j}(s,a) -\beta\log\pi_\psi(a|s)\Big],
\end{aligned}
\end{equation}
which includes an SAC-style entropy regularization.
We summarize the training process in Algorithm ~\ref{alg}.

\begin{algorithm}[t]
\caption{Diverse randomized value functions for offline RL (DRVF)}
\label{alg}
\begin{algorithmic}[1]
\REQUIRE Offline dataset $\cD_m$, policy network $\pi_{\psi}$, 
$Q$-networks $\{Q_{w_j}\}_{j\in[M]}$, and target $Q$-networks $\{Q_{w^-_j}\}_{j\in[M]}$. 
\STATE Initialize the parameters, including $\psi$, $\{w_j\}_{j\in[M]}$, and $\{w^-_j\}_{j\in[M]}$.
\WHILE{\emph{not coverage}}
\STATE Sample transitions $\{(s,a,r,s')\}$ from $\cD_m$ and construct the OOD data as $\{(s,a^{\rm ood})\}$.
\STATE Calculate the pessimistic target $Q$-values by sampling from the posterior via Eq. (\ref{eq:parc-lcb}).
\STATE Minimize the critic loss $\mathcal{L}_{\rm critic}$ in Eq.~\eqref{eq:critic} and update the parameters.
\STATE Train the pessimistic policy by minimizing the actor loss $\mathcal{L}_{\rm actor}$ in Eq. (\ref{eq:Lp}).
\STATE Update target $Q$-networks by $\{w^-_j\}'\leftarrow \rho \{w^-_j\}' + (1-\rho)\{w_j\}$
\ENDWHILE
\end{algorithmic}
\end{algorithm}

\begin{proposition}[\textbf{Informal}]
Optimizing the loss function in Eq. (\ref{eq:critic}) with the repulsive term can obtain the true Bayesian posterior of the value function.
\label{gradient-flow}
\end{proposition}

This formal version of Proposition \ref{gradient-flow} is shown in Appendix B, where we establish the equivalence between the repulsive term and an RBF kernel and relate the optimization of Eq. (\ref{eq:critic}) to the gradient flow in Wasserstein space.

\section{Experiments}
In this section, we perform multiple experiments and comparisons across various setups to explore the following questions: 1) Does our method exhibit better performance than other baseline methods in different setups, including continuous control tasks and navigation tasks? 2) Can our method obtain reasonable uncertainty estimation for the value function? 3) What is the computational or parametric efficiency of DRVF? 4) Which parts or factors are critical in our method? 
We evaluate our algorithm using the standard D4RL benchmarks.

\begin{table*}[!t]
\caption{Performance of baseline algorithms and our proposed method on Gym Mujoco tasks. The two highest scores are in bold. The abbreviations `r', `m', `m-r', `m-e', `e' correspond to random, medium, medium-replay, medium-expert, and expert datasets. The subscripts $\heartsuit$, $\diamondsuit$, and $\clubsuit$ indicate model-free, model-based, and uncertainty-based methods.}
\resizebox{\textwidth}{!}{
\begin{tabular}{lrrrrrrrrrrr}
\toprule
Environments & CQL$_\heartsuit$ & TD3-BC$_\heartsuit$ & IQL$_\heartsuit$ & MOPO$_\diamondsuit$ & COMBO$_\diamondsuit$ & ATAC$_\heartsuit$ & UWAC$_{\heartsuit \clubsuit}$ & PBRL$_{\heartsuit \clubsuit}$ & EDAC$_{\heartsuit \clubsuit}$ & SAC-N$_{\heartsuit \clubsuit}$ & DRVF$_{\heartsuit \clubsuit}$ \\ 
\midrule
halfcheetah-r & 17.5$\pm$1.7 & 10.8±1.6 & 12.6±2.1 & \textbf{35.9±2.9} & 38.8$\pm$3.7 & 4.8 & 2.3$\pm$ 0.0 & 13.1±1.2 & 27.7$\pm$1.0 & 28.0$\pm$0.9 & 30.9$\pm$1.1 \\
halfcheetah-m & 47.0±0.5 & 48.3±0.3 & 47.3±0.3 & \textbf{73.1±2.4} & 54.2$\pm$1.5 & 54.3 & 42.2$\pm$0.4 & 58.2±1.5 & 66.0$\pm$1.5 & 67.5$\pm$1.2 & \textbf{69.2$\pm$2.7} \\
halfcheetah-m-r & 45.5±0.8 & 44.6±0.3 & 44.4±0.5 & \textbf{69.2±1.1} & 55.1$\pm$1.0 & 49.5 & 35.9$\pm$3.7 & 49.5±0.8 & 62.3$\pm$1.4 & 63.9$\pm$0.8 & \textbf{66.7$\pm$2.0} \\
halfcheetah-m-e & 75.6±28.7 & 90.1±4.4 & 89.3±3.0 & 70.3±21.9 & 90.0$\pm$5.6 & 95.5 & 42.7$\pm$0.3 & 93.6±2.3 & \textbf{107.7$\pm$2.1} & \textbf{107.1$\pm$2.0} & 104.2$\pm$2.4 \\
halfcheetah-e & 96.3±1.4 & 96.7±0.9 & 95.1±0.3 & 81.3±21.8 & - & 94.0 & 92.9$\pm$0.6 & 96.2±2.3 & \textbf{106.6$\pm$1.2} & 105.2$\pm$2.6 & \textbf{107.1$\pm$2.6} \\
walker2d-r & 5.1±1.5 & 0.8±0.7 & 6.0±1.1 & 4.2±5.7 & 7.0$\pm$3.6 & 8.0 & 2.0$\pm$0.4 & 8.8±6.3 & \textbf{17.6$\pm$9.7} & \textbf{21.7$\pm$0.0} & 7.4$\pm$5.7 \\
walker2d-m & 73.3±20.0 & 85.0±1.0 & 75.6±2.3 & 41.2±30.8 & 81.9$\pm$2.8 & 91.0 & 75.4$\pm$3.0 & 90.3±1.2 & \textbf{93.7$\pm$1.3} & 87.9$\pm$0.2 & \textbf{95.1$\pm$2.7} \\
walker2d-m-r & 81.8±3.0 & 80.6±9.1 & 67.0±12.2 & 73.7±9.4 & 56.0$\pm$8.6 & \textbf{94.1} & 23.6$\pm$6.9 & 86.2±3.4 & 85.7$\pm$2.2 & 78.7$\pm$0.7 & \textbf{96.1$\pm$2.3}\\
walker2d-m-e & 107.9±1.8 & 110.4±0.6 & 109.4±0.7 & 77.4±27.9 & 103.3$\pm$5.6 &  \textbf{116.3} & 96.5$\pm$9.1 & 109.8±0.2 & 113.5$\pm$1.1 & \textbf{116.7$\pm$0.4} & 112.4$\pm$1.1 \\
walker2d-e & 108.6±0.6 & 110.2±0.5 & 109.8±0.1 & 62.4±3.2 & - & 108.2 & 108.4$\pm$0.4 & 108.8±0.2 & \textbf{115.4$\pm$1.5} & 107.4$\pm$2.4 & \textbf{112.5$\pm$0.5} \\
hopper-r & 7.9±0.4 & 9.0±0.3 & 8.1±0.9 & 16.7±12.2 & 17.9$\pm$1.4 &  \textbf{31.8} & 2.7$\pm$0.3 & \textbf{31.6±0.3} & 12.6$\pm$10.5 & 31.3$\pm$0.0 & 27.9$\pm$7.8 \\
hopper-m & 53.0±31.8 & 59.0±4.0 & 67.9±6.4 & 38.3±34.9 & 97.2$\pm$2.2 & \textbf{102.8} & 50.9$\pm$4.4 & 81.6±14.5 & 101.9$\pm$0.4 & 100.3$\pm$0.3 & \textbf{102.0$\pm$1.2} \\
hopper-m-r & 88.7±14.4 & 52.3±23.5 & 94.8±4.9 & 32.7±9.4 & 89.5$\pm$1.8 & \textbf{102.8} & 25.3$\pm$1.7 & 100.7±0.4 & 100.8$\pm$0.1 & 101.8$\pm$0.5 & \textbf{102.9$\pm$1.1} \\
hopper-m-e & 105.6±14.4 & 102.9±7.0 & 90.8±14.8 & 60.6±32.5 & 111.1$\pm$2.9 & \textbf{112.6} & 44.9$\pm$8.1 & 111.2±0.7 & 110.6$\pm$0.2 & 110.1$\pm$0.3 & \textbf{112.3$\pm$1.8} \\
hopper-e & 96.5±31.4 & 108.9±2.5 & 109.1±3.2 & 62.5±29.0 & - & \textbf{111.5} & 110.5$\pm$0.5 & 110.4±0.3 & 109.8$\pm$0.1 & 110.3$\pm$0.3 & \textbf{111.4$\pm$1.2} \\
\midrule
Average & 67.4±10.2 & 67.3±3.8 & 68.5±3.5 & 53.3±16.3 & 66.8$\pm$3.4 & 78.5 & 50.4$\pm$2.7 & 76.7±2.4 & \textbf{82.9$\pm$2.2} & 82.5$\pm$0.8 & \textbf{83.9$\pm$2.4} \\ 
\bottomrule
\end{tabular}}
\label{table1}
\end{table*}

\subsection{Performance on D4RL benchmarks}
\subsubsection{Baselines}

To verify the feasibility and superiority of our algorithm, we compare DRVF against different types of offline RL methods: 1) model-free value-based methods, including CQL \cite{cql-2020}, TD3-BC \cite{td3bc-2021}, and IQL \cite{iql}; 2) model-based algorithms such as MOPO \cite{mopo-2020} and COMBO \cite{combo-2021}, which use ensembles to learn transition dynamics; 3) ATAC \cite{ATAC-2022}, which exhibits promising results by framing offline RL as a two-player game;
4) uncertainty-based methods, including UWAC \cite{UWAC-2020}, PBRL \cite{PBRL-2022}, EDAC \cite{EDAC-2021}, and SAC-N. In particular, we compare DRVF with uncertainty-based methods to show that DRVF obtains reliable uncertainty estimation and achieves superior performance with fewer ensemble networks.

\subsubsection{Results in Gym-Mujoco tasks}

We evaluate the performance of all methods on Gym Mujoco tasks, which include three environments: halfcheetah, walker2d, and hopper. Each environment consists of five datasets (random, medium, medium-replay, medium-expert, and expert) generated by different behavior policies.
We report the normalized scores for each task and the average scores in Gym Mujoco tasks in Table \ref{table1}.

The experimental results in Table \ref{table1} show that DRVF outperforms CQL, TD3-BC, IQL, MOPO, and COMBO by a wide margin in most cases and showcases better performance than ATAC in the halfcheetah and walker2d environments.
When compared to uncertainty-based methods such as UWAC, PBRL, or SAC-N, DRVF performs better across most tasks while providing reliable uncertainty quantification.
Although DRVF is slightly worse than EDAC in halfcheetah/walker2d-medium-expert and walker2d-expert tasks, it obtains reasonable scores with much fewer ensembles, which will be discussed in Section \ref{computational-efficiency}. We provide more comparisons in Appendix D.

\begin{figure*}[t]
    \centering
    \subfigure[halfcheetah]{\includegraphics[width=0.32\linewidth]{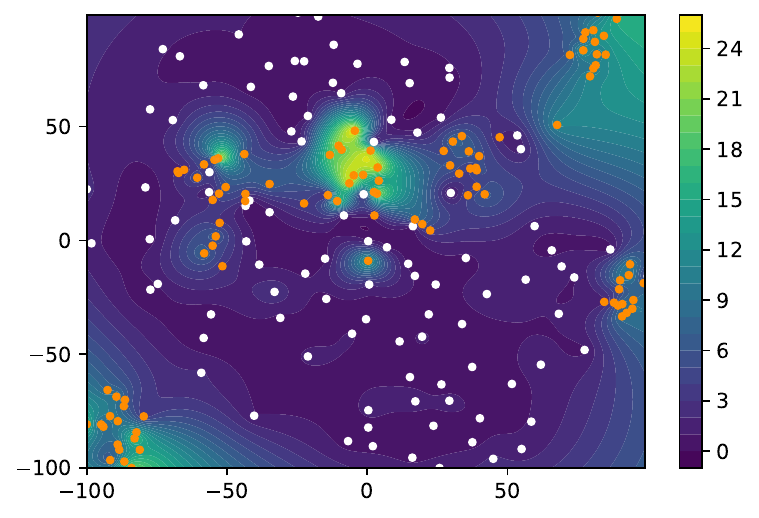}}
    \subfigure[walker2d]{\includegraphics[width=0.32\linewidth]{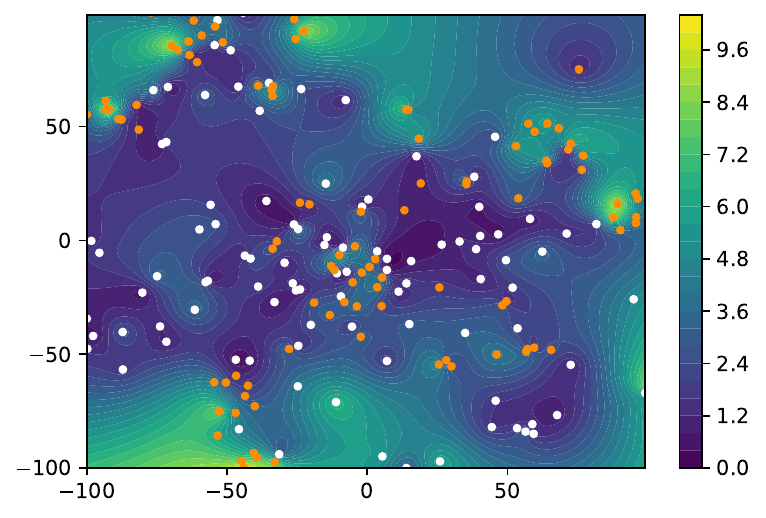}}
    \subfigure[hopper]{\includegraphics[width=0.32\linewidth]{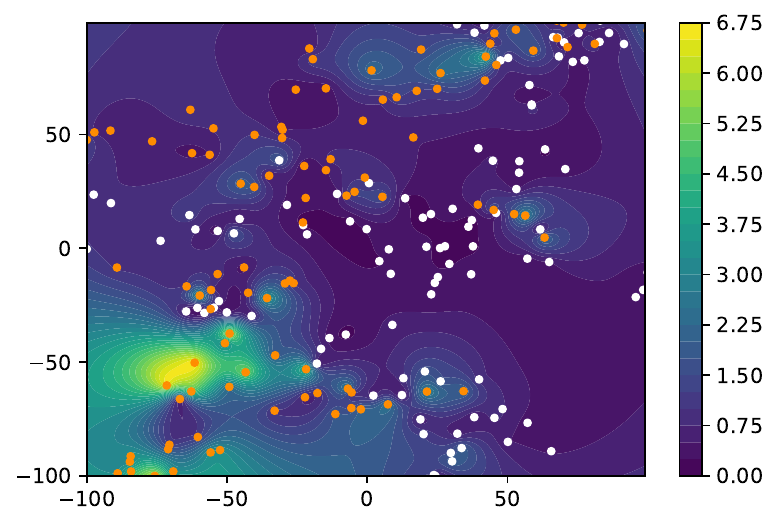}}

    \subfigure[halfcheetah]{\includegraphics[width=0.32\linewidth]{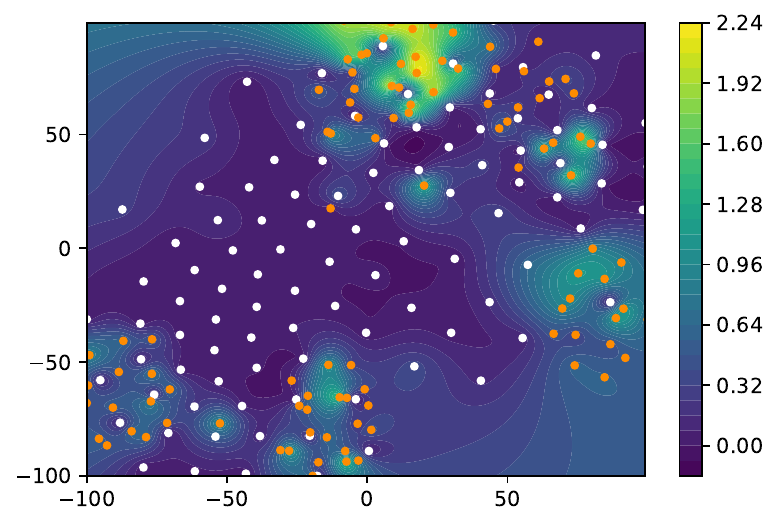}}
    \subfigure[walker2d]{\includegraphics[width=0.32\linewidth]{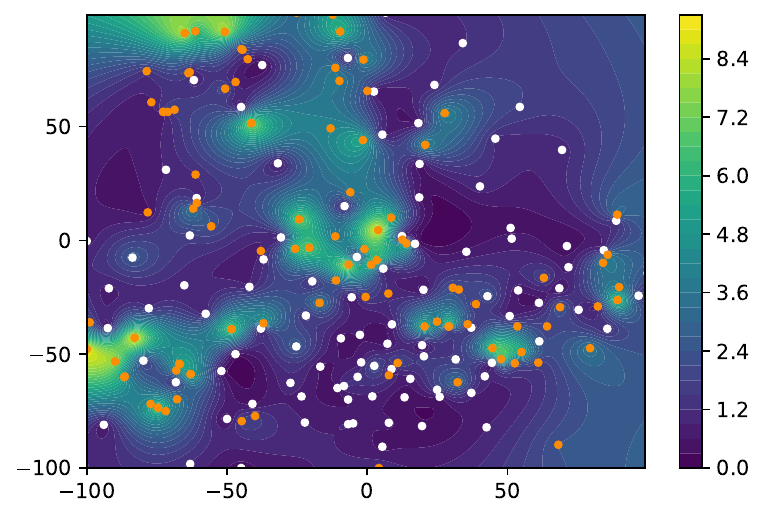}}
    \subfigure[hopper]{\includegraphics[width=0.32\linewidth]{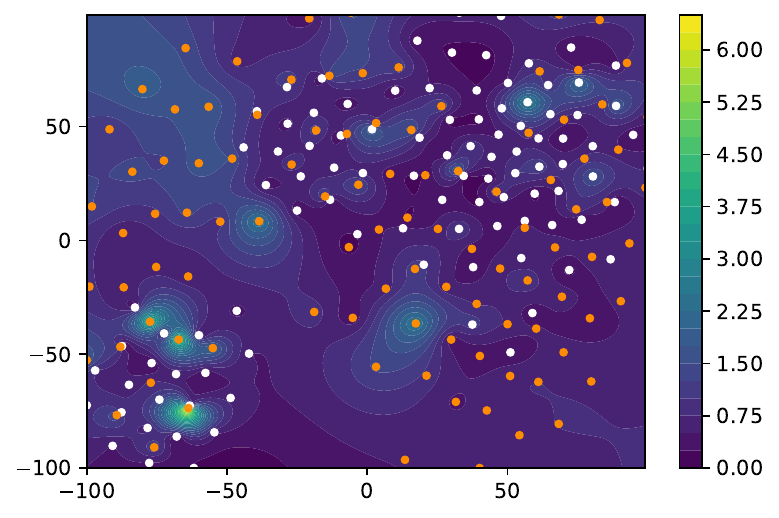}}
    \caption{Uncertainty estimation of the in-distribution samples (white) and OOD samples (orange). The results (a-c) are obtained from DRVF, while the results (d-f) are estimated by PBRL. The brighter area in the contour indicates higher uncertainty, while the darker area indicates lower uncertainty. DRVF demonstrates a more favorable alignment of the brighter areas with the OOD samples, suggesting an enhanced ability to quantify uncertainty.}
    \label{uncertainty_estimation}
\end{figure*}

\subsection{Uncertainty Quantification}

To validate whether DRVF provides a reasonable uncertainty quantification, especially compared with PBRL, we visualize the uncertainty estimation of $Q$-values for in-distribution and OOD samples in Fig. \ref{uncertainty_estimation}. Specifically, we train the Bayesian posterior in the medium-replay dataset and evaluate the learned uncertainty for in-distribution and OOD data. Since expert trajectories are not contained in the medium-replay dataset, the state-action pairs induced by the expert policy can be regarded as OOD data points. Therefore, we sample in-distribution data points from the medium-replay dataset and OOD data points from the expert dataset. The original states and actions are projected into two-dimensional space using t-SNE. The uncertainty quantification for different state-action pairs is visualized in the contour plot.

As depicted in Fig. \ref{uncertainty_estimation}, in-distribution samples (white) mainly distribute in darker areas, indicating lower uncertainty. In contrast, OOD samples (orange) are situated in brighter areas and assigned with higher uncertainty.
% We find the uncertainty estimation generalizes well from in-distribution areas to OOD areas, providing reasonable penalties for OOD actions without leading to overly conservative policies. 
It is expected that $Q$ values exhibit low uncertainty on in-distribution data and high uncertainty in the OOD region. In this regard, DRVF demonstrates the desired behavior, as the bright areas in the contour plot align with the OOD region. On the other hand, PBRL fails to accurately estimate the uncertainty, as indicated by the inconsistency between the OOD region and the contour plot. By providing a comparison with PBRL, we emphasize the superior uncertainty quantification capability of DRVF, validating our motivation to improve uncertainty estimation in offline RL.

\subsection{Computational Efficiency}\label{computational-efficiency}

\begin{figure*}[ht]
    \centering
\includegraphics[width=0.32\linewidth]{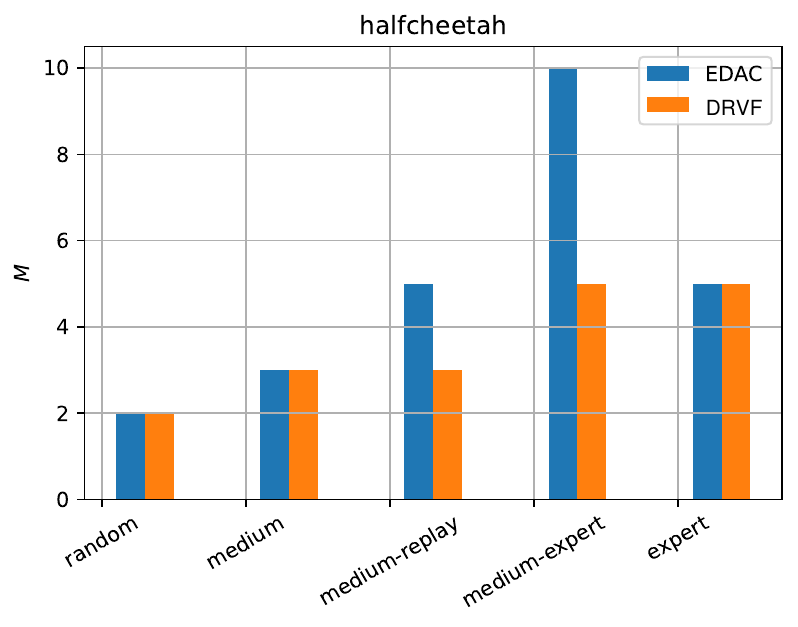}
\includegraphics[width=0.32\linewidth]{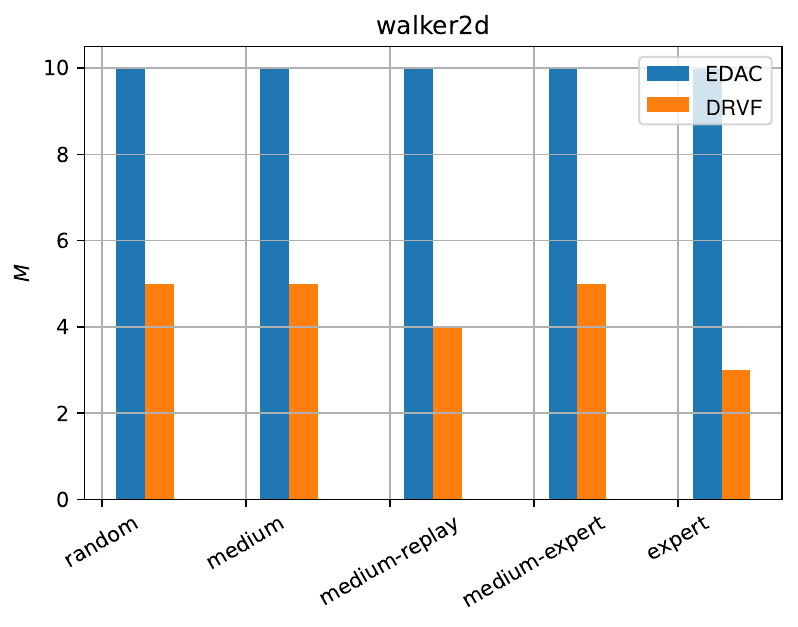}
\includegraphics[width=0.32\linewidth]{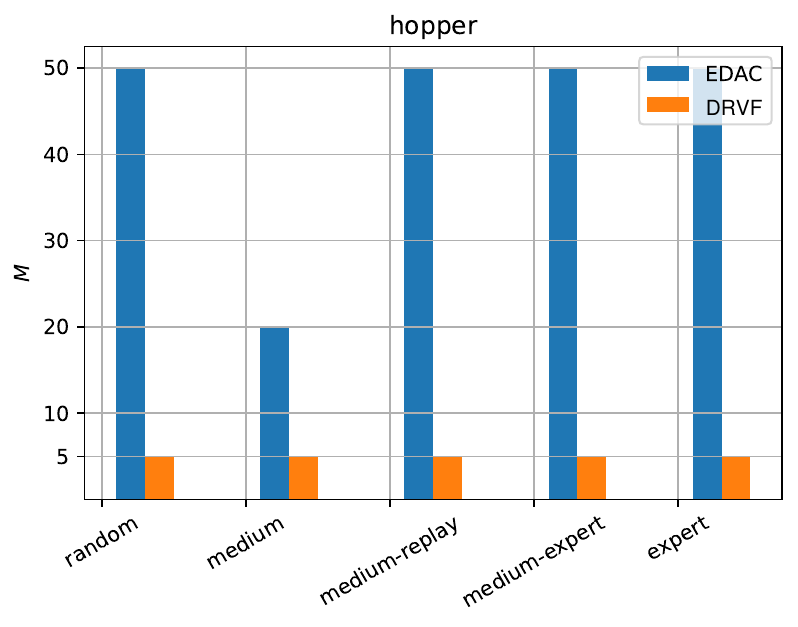}
    \caption{Minimum number of $Q$-ensembles ($M$) required to obtain the performance in Table \ref{table1}. DRVF needs much fewer ensembles than EDAC in most cases and reduces the required parameters.}
    \label{ensemble_number}
\end{figure*}

\begin{table}[ht]
\centering
\caption{Computational costs of uncertainty-based methods.}
\footnotesize
\label{computation}
\resizebox{0.6\textwidth}{!}{
\begin{tabular}{llrrrr}
\toprule
& Environments & PBRL & SAC-N & EDAC & DRVF \\ 
\midrule
\multirow{3}{*}{\makecell{Runtime \\ (s/epoch)}} & halfcheetah-m & 147.3 & 19.4 (N=10)   &  24.9    &   33.4   \\
& walker2d-m    & 163.7 & 17.5(N=20)   &   24.5   &   33.8   \\
& hopper-m   & 150.7 & 117.1 (N=500)  &   25.5   &    33.1  \\ 
\midrule
\multirow{3}{*}{\makecell{Number of \\ parameters}} & halfcheetah-m & 2.9M & 2.9M (N=10) & 2.9M & \textbf{1.5M} \\
& walker2d-m & 2.9M & 5.7M (N=20) &  2.9M & \textbf{1.5M}\\
& hopper-m & 2.8M & 135.8M (N=500) & 13.7M & \textbf{1.5M} \\ 
\bottomrule
\end{tabular}}
\end{table}
We evaluate the computational efficiency using two criteria: the runtime per epoch and the number of parameters. We compare our method against other uncertainty-based methods, including PBRL, SAC-N, and EDAC. All methods are tested on the same physical machine equipped with a Tesla P40 GPU. As presented in Table \ref{computation}, PBRL takes much longer time due to its lack of parallel processing for the ensemble. In addition, EDAC takes more time than SAC-N to calculate the diversified gradient regularization, and DRVF spends time sampling from the posterior distribution and OOD data. We remark that although DRVF takes a slightly longer time than EDAC or SAC-N (with N=10 or 20), it offers a significant improvement in parametric efficiency. For instance, in the halfcheetah and walker2d tasks, DRVF needs half the number of parameters of EDAC or PBRL. In the hopper environments, EDAC or SAC-N rely on a large number of networks, while DRVF only needs about one-tenth or one-percent of their parameters. In Fig. \ref{ensemble_number}, we also illustrate the number of ensembles necessary to achieve the strong performance of EDAC and DRVF in Table \ref{table1}.

\subsection{Ablation Studies}

In this part, we examine the effect of three main components in our method: BNN, repulsive regularization term, and ensemble. In general, we find that Bayesian inference module is necessary for our method, and the repulsive regularization by sampling OOD actions from the learned policy is crucial for policy learning. Meanwhile, we observe that more ensembles can lead to more stable performance and that DRVF can obtain favorable performance with 5 or fewer ensembles. DRVF is also insensitive to the number of samples from the posterior distribution and OOD actions. Due to page limits, we put detailed analyses in Appendix E.

\section{Conclusion}
In this paper, we introduce DRVF, a lightweight and efficient uncertainty-based offline RL algorithm. We use Bayesian neural networks combined with ensembles to construct randomized value functions and estimate the Bayesian posterior and uncertainties for $Q$-values by sampling from the ensemble BNNs. Under the linear MDP assumptions, DRVF recovers the provably efficient LCB penalty and the true Bayesian posterior. By incorporating repulsive regularization, DRVF encouranges diversity among randomzied value functions and achieves significant performance improvements with fewer ensembles. We validate the performance and parametric efficiency of DRVF on the D4RL benchmarks, including challenging AntMaze tasks. 
Future works will extend this idea to learning environmental dynamics models and exploring the online fine-tuning capability. It is also expected to empirically reduce the time required for sampling and improve convergence and robustness, especially in sparse-reward settings.

\section{Acknowledgements}
This work was supported by the Heilongjiang Touyan Innovation Team Program.

\bibliographystyle{unsrt}  
\bibliography{references}

\clearpage
\appendix
% \appendices
\section{Theoretical Proof}
\label{app:them}
\subsection{Proof of the LCB-Term}
To learn the posterior distribution of $Q$-function, we consider a Bayesian perspective of LSVI and use Bayesian linear regression to learn the posterior of $w$. Under linear MDP assumptions, the following theorem establishes the method to approximate the LCB-penalty $\Gamma_t^{\rm lcb}(s_t,a_t)$.
\begin{theorem*}
[Theorem 1 restate]
Under linear MDP settings, it holds for the standard deviation of the estimated posterior distribution $\mathbb P(\tilde Q \,|\, s, a, \cD_m)$ follows
\begin{equation}
\label{eq:lcb-1}
\mathrm{Std}(\tilde Q \,|\, s_t, a_t, \cD_m) 
= \left[\psi(s_t, a_t)^\top \Lambda^{-1}_t \psi(s_t, a_t)\right]^{\nicefrac{1}{2}}
:= \Gamma_t^{\rm lcb}(s_t,a_t),
\end{equation}
where $\Lambda_t = \sum\nolimits_{i \in [m]} \psi(s^i_t, a^i_t) \psi(s^i_t, a^i_t)^\top + \lambda \cdot \mathbf I$, and $\Gamma_t^{\rm lcb}(s_t,a_t)$ is defined as the LCB-term.
\end{theorem*}
\begin{proof}
    It holds from Bayes rule that 
    \begin{equation}\nonumber
        \log q^*(w \,|\, \mathcal D_m) 
        = \log \mathbb P(w \,|\, \mathcal D_m)
        = \log \mathbb P(w) + \log \mathbb P(\mathcal D_m \,|\, w) + \mathrm{Const}.
    \end{equation}
    By assuming that the prior of $w$ is the standard normal distribution, we have
    \begin{equation}
        Q \,|\, s, a, w \sim \mathcal N(\psi(s, a)^\top w, 1).
    \end{equation}
    Thus, it holds that
    \begin{equation}
        \log q^*(w \,|\, \mathcal D_m) 
        = - \frac{\lambda}{2} \| w \|^2 - \sum_{i \in [m]} \frac{1}{2} \| \psi(s^i_t, a^i_t)^\top - y^i_t \|^2 + \mathrm{Const}
        = -\frac{1}{2} (w - \mu_t)^\top \Lambda_t^{-1} (w - \mu_t) + \mathrm{Const}.
    \end{equation}
    where we have 
    \begin{equation}
        \mu_t = \Lambda_t^{-1} \sum_{i \in [m]} \psi(s^i_t, a^i_t) y_t^i, 
        \Lambda_t = \sum\nolimits_{i \in [m]}
        \psi(s^i_t, a^i_t) \psi(s^i_t, a^i_t)^\top + \lambda \cdot \mathbf I
    \end{equation}
    Thus, we obtain that $w \,|\, \cD_m \sim \mathcal N(\mu_t, \Lambda_t^{-1})$. It then holds for any $(s, a) \in \cS \times \cA$ that
    \begin{equation}
        \mathrm{Var}(\tilde Q \,|\, S=s, A=a) 
        = \mathrm{Var}(\psi(s, a)^\top w)
        = \psi(s, a)^\top \Lambda^{-1}_t \psi(s, a),
    \end{equation}
    which concludes the proof of Theorem 1.
\end{proof}

\subsection{$\xi$-Uncertainty Quantifier and the optimality gap}

In this part, we first recall the definition of a $\xi$-uncertainty quantifier, and then prove that the LCB-penalty $\Gamma_t^{\rm lcb}(s_t,a_t)$ obtained by ensemble BNNs is a $\xi$-uncertainty quantifier and benefits the optimality of the value function.
\begin{definition*}
[Definition 1 restate]
% ($\xi$-Uncertainty Quantifier (\cite{pevi-2021})).
For all $(s,a)\in{\mathcal{S}}\times\mathcal{A}$, if the inequality 
$|\widehat{\mathcal{T}}V_{t+1}(s,a)-\mathcal{T}V_{t+1}(s,a)|\leq\Gamma_{t}(s,a)$
holds with probability at least $1-\xi$, then $\{\Gamma_{t}\}_{t\in[T]}$ is a $\xi$-Uncertainty Quantifier. $\widehat{\mathcal{T}}$ empirically estimates $\mathcal{T}$ based on the data.
% The set of penalization $\{\Gamma_{t}\}_{t\in[T]}$ forms a $\xi$-Uncertainty Quantifier if it holds with probability at least $1-\xi$ that
% $$|\widehat{\mathcal{T}}V_{t+1}(s,a)-\mathcal{T}V_{t+1}(s,a)|\leq\Gamma_{t}(s,a)$$
% for all $(s,a)\in{\mathcal{S}}\times\mathcal{A}$, where $\mathcal{T}$ is the Bellman operator and $\widehat{\mathcal{T}}$ is the empirical Bellman operator that estimates $\mathcal{T}$ based on the data.
\end{definition*}
Under the linear MDP assumption, the value estimation $\widehat{\mathcal{T}}V_{t+1}$ is estimated by solving the least-squares problem in Eq. (3). Therefore, the value estimation can be expressed by $\widehat{\mathcal{T}}V_{t+1}(s_t, a_t)=\psi(s_t,a_t)^\top\widehat{w}_t$. We remark that the $\xi$-uncertainty quantifier are widely used in the theoretical analysis for online RL and offline RL \cite{lsvi-2020,pevi-2021, xie2021bellman}, and $\Gamma_t^{\rm lcb}(s,a)$ in Eq. (\ref{eq:lcb-1}) forms a valid $\xi$-uncertainty quantifier.
\begin{proposition*}[Proposition 1 restate]
    The LCB-penalty in Eq. (\ref{eq:lcb-1}) forms a valid $\xi$-uncertainty quantifier.
\end{proposition*}
\begin{proof}
    Based on the solution 
    \begin{equation}
        \widehat w_t = \Lambda^{-1}_t\sum^m_{k = 1}\psi(s^k_t, a^k_t)y_t^k,
        \Lambda_t = \sum^m_{k=1}\psi(s^k_t, a^k_t)\psi(s^k_t, a^k_t)^\top,
    \end{equation}
    the distance between the empirical Bellman operator and Bellman operator satisfies 
    \begin{equation}
        \begin{aligned}
        \widehat{\mathcal{T}}V_{t+1}(s_t, a_t) 
        - \mathcal{T}V_{t+1}(s_t, a_t)
        &= \psi(s_t,a_t)^\top\left (\widehat{w}_t-w_t)w_t \right )\\
        &= \psi(s_t,a_t)^\top \left (\Lambda^{-1}_t\sum^m_{k = 1}\psi(s^k_t, a^k_t)y_t^k - \Lambda^{-1}_t\Lambda_t w_t \right )\\
        &= \psi(s_t,a_t)^\top \Lambda^{-1}_t \left (\sum^m_{k = 1}\psi(s^k_t, a^k_t)y_t^k - \sum^m_{k=1}\psi(s^k_t, a^k_t)\psi(s^k_t, a^k_t)^\top w_t \right )\\
        &= \psi(s_t,a_t)^\top \Lambda^{-1}_t\sum^m_{k = 1}\psi(s^k_t, a^k_t) 
        \left (r(s^k_t, a^k_t)+ V_{t+1}(s^k_{t+1})  - \mathcal{T}V_{t+1}(s^k_t, a^i_t) \right)
        \end{aligned}
    \end{equation}
    Following the standard analysis in \cite{lsvi-2020, pevi-2021, PBRL-2022}, it holds with probability at least $1-\xi$ that 
    \begin{equation}
        |\psi(s_t,a_t)^\top \Lambda^{-1}_t\sum^m_{k = 1}\psi(s^k_t, a^k_t) \left (y_t^k - \mathcal{T}V_{t+1}(s^k_t, a^k_t)\right )|
        \leq \tau_t \cdot 
        \sqrt{\psi(s_t, a_t)^\top \Lambda^{-1}_t \psi(s_t, a_t)},
        % \left[\psi(s_t, a_t)^\top \Lambda^{-1}_t \psi(s_t, a_t)\right]^{\nicefrac{1}{2}},
    \end{equation}
    where $ \tau_{t}=\mathcal{O}\bigl{(}T\cdot\sqrt{d}\cdot\text{log}(T/\xi)\bigr{)} > 0$ is the scaling parameter. Therefore, we obtain that 
    \begin{equation}
        |\widehat{\mathcal{T}}V_{t+1}(s,a)-\mathcal{T}V_{t+1}(s,a)|\leq\tau_t \cdot 
        \sqrt{\psi(s_t, a_t)^\top \Lambda^{-1}_t \psi(s_t, a_t)}.
        % \left[\psi(s_t, a_t)^\top \Lambda^{-1}_t \psi(s_t, a_t)\right]^{\nicefrac{1}{2}}.
    \end{equation}
    % for all $(s,a)\in{\mathcal{S}}\times\mathcal{A}$ with probability at least $1-\xi$.
\end{proof}

Since the LCB-penalty based on ensemble BNNs forms a $\xi$-uncertainty quantifier, we can further characterize the optimality gap of the value iteration.
\begin{corollary*}[Corollary 1 restate]
(\cite{pevi-2021}). 
Since $\Gamma_{t}^{\rm lcb}(s_{t},a_{t})$ is a $\xi$-uncertainty quantifier, under Definition 1,
if we set $y=\mathcal{T}V_{t+1}(s,a)$, then the suboptimality gap satisfies
$$V^{*}(s_{1})-V^{\pi_{1}}(s_{1})\leq\sum^{T}_{t=1}\mathbb{E}_{\pi^{*}}\bigl{[}%
\Gamma_{t}^{\rm lcb}(s_{t},a_{t}){\,|\,}s_{1}\bigr{]}.$$
\end{corollary*}
\begin{proof}
    See e.g., \cite{pevi-2021} for a detailed proof.
\end{proof}
This corollary implies that the pessimism of the value function holds with probability at least
$1-\xi$ as long as
$\Gamma_{t}^{\rm lcb}(s_{t},a_{t})$ 
are $\xi$-uncertainty quantifiers.

\section{Repulsive regularization for approximating the Bayesian posterior}\label{app:repulsive}

In this section, we relate the repulsive term with recent theoretical work about repulsive ensembles. Our analysis demonstrates that the repulsive term in Eq. (9) can efficiently promote diversity and encourage the ensemble BNNs to present better uncertainty quantification with small ensembles.

Firstly, we suggest that the repulsive term can be expressed in terms of the gradients of a kernel and that the resulting update of the Q-function is similar to computing the gradient flow dynamics for estimating the true Bayesian posterior through optimizing the KL divergence in Wasserstein space.
\begin{lemma}
The repulsive term in Eq. (9), i.e. the standard deviation ${\rm Std}\big(\{Q^{(i)}\}_{i\in[n]}\big)$ is a function of the sum of the RBF kernel 
\begin{equation}
    \sum_{i,j\in [n]}k(Q^{(i)},Q^{(j)})=\sum_{i,j\in [n]}\exp\big{(}-\gamma(Q^{(i)}-Q^{(j)})^{2}\big{)},
\end{equation}
where $Q^{(i)}, i\in[n]$ and $Q^{(j)},j\in[n]$ are estimates from ensemble BNNs, and $\gamma$ is the parameter of the RBF kernel. 
\label{lemma1}
\end{lemma}
\begin{proof}
    In \cite{d2021repulsive}, the RBF kernel is used to model the interaction between ensemble members and create a repulsive effect. We present the relationship between the standard deviation and the RBF kernel as follows:
    \begin{equation}
    \begin{aligned}
    {\rm Std}\big(\{Q^{(i)})\}_{i\in[n]}\big) 
    \propto {\rm Var}\big(\{Q^{(i)})\}_{i\in[n]}\big) = \frac{1}{n^2}\sum_{i=1}^n\sum_{j=1}^n \frac{(Q^{(i)}-Q^{(j)})^2}{2}
    = -\frac{1}{2n^2\gamma}\sum_{i,j=1}^n \log k(Q^{(i)}, Q^{(j)})
    \end{aligned}
    \end{equation}
\end{proof}

Inspired by \cite{d2021repulsive}, we first introduce the gradient flow in function space and then show that optimizing the loss function in Eq. (10) can be thought as an equivalence to the gradient flow dynamics for estimating the true Bayesian posterior by minimizing KL divergence in Wasserstein space. We assume that the true Bayesian posterior of the Q-function is $\mathbb{P}(Q|s,a,\mathcal{D}_m)$ denoted by $p(Q)$, the learned Q-function is $p(\tilde{Q})$, and $Q^{(i)}\sim p(\tilde{Q})$. To make $p(\tilde{Q})$ close to the true Bayesian posterior, we can minimize the KL divergence between them: 
\begin{equation}\nonumber
\begin{aligned}
    \inf_{p(\tilde{Q})\in\mathcal{P}(\mathcal{M})}D_{KL}(p(\tilde{Q}),p(Q))
    =\int_{\mathcal{M}}\left( \log p(\tilde{Q}) -\log p(Q)\right )p(\tilde{Q})\,\mathrm{d}\tilde{Q}\;,
\end{aligned}
\end{equation}
where $\mathcal{P}(\mathcal{M})$ is the space of probability measures and $\mathcal{M}$ is a Riemannian manifold. This objective can be minimized by building a path along the negative gradient, and the \emph{gradient flow} mean s the curve $Q(t)$ described by that path.

\begin{proposition*}[Formal version of Proposition 2]
The optimization for the loss function of the critic network
\begin{equation}
\begin{aligned}
\cL_{\rm critic}^j \,=\, 
\eta_q \cdot \EE_{s,a,r,s'\sim \cD_m}\Big[\Big(Q_{w_j}(s,a)-\hT^{\rm lcb} Q_{w_j}(s,a)\Big)^2
+ \mathrm{KL}\big(q_\theta (w_j) \,\|\, \mathbb P(w)\big)\Big] \\
-\eta_{\rm ood} \cdot \EE_{s\sim \cD_m,a^{\rm ood}\sim \pi}\,[\,\cR\big(Q_{w_j}(s, a^{\rm ood})\big)\,]\,,
\end{aligned}
\label{eq:app-critic}
\end{equation}
is approximate to the update rule for the gradient flow dynamics. Therefore, optimizing the loss function can obtain the true Bayesian posterior.
\end{proposition*}
\begin{proof}
The evolution in time of $p(\tilde{Q})$, denoted by the \emph{Wasserstein gradient flow}, can be written as follows:
\begin{equation}
    \begin{split}\frac{\partial p(\tilde{Q})}{\partial t}=\nabla\cdot\bigg{(}%
p(\tilde{Q})\nabla\frac{\delta}{\delta p(\tilde{Q})}D_{KL}(p(\tilde{Q}),p(Q))\bigg{)}=\nabla\cdot\bigg{(}p(\tilde{Q})\nabla\big{(}\log p(\tilde{Q})-\log p(Q)\big{)}\bigg{)}\,.\end{split}
\end{equation}
where $\nabla\frac{\delta}{\delta p(\tilde{Q})}D_{KL}(p(\tilde{Q}),p(Q))$ is the Wasserstein gradient and $\frac{\delta}{\delta p(\tilde{Q})}:\mathcal{P}(\mathcal{M})\to\mathbb{R}$ indicates the functional derivative. Following this update, we can give the mean field functional dynamics:
\begin{equation}
    \frac{d{Q}}{dt}=-\nabla\big{(}\log p(\tilde{Q})-\log p(Q)\big{)}\,.
\end{equation}
A sample-based kernel density estimation (KDE) with an RBF kernel can be regarded as smooth estimation for the empirical data distribution and used to approximate the gradient $\nabla \log p(\tilde{Q})$ \cite{d2021repulsive}, that is,
\begin{equation}
    \nabla\log p(\tilde{Q})\approx\frac{\sum_{j=1}^{n}\nabla_{Q_t^{(i)}}k(Q^{(i)}_{t},Q^{(j)}_{t})}{\sum_{j=1}^{n}k(Q^{(i)}_{t},Q^{(j)}_{t})}.
\end{equation}
In this way, we can express the evolution in a discretized manner:
\begin{equation}
    {Q}_{t+1}^{i}={Q}_{t}^{i}+\epsilon_{t}\bigg{(}\nabla_{{Q}}\log%
p(Q_{t}^{i})-\frac{\sum_{j=1}^{n}\nabla_{{Q}^{(i)}_{t}}k({Q}^%
{(i)}_{t},{Q}^{(j)}_{t})}{\sum_{j=1}^{n}k({Q}^{(i)}_{t},{Q}^{(j)}_{t})}%
\bigg{)}\,,
\label{eq:discretized}
\end{equation}
where $\epsilon_t$ is the step size. As shown in \cite{d2021repulsive}, the update rule not only encourages the ensemble members to provide diverse predictions but also helps to converge to the true Bayesian posterior.

    As presented in Eq. (4), the first two term in Eq. (\ref{eq:app-critic}) aims to approximate the posterior of the value function $Q$ by maximizing the ELBO, which plays the same role of the term $\log p(\tilde{Q})$ in Eq. (\ref{eq:discretized}). 
    % can be optimized . 
    
    On the other hand, Lemma \ref{lemma1} shows that the standard deviation of Q-values sampled from ensemble BNNs can be propositional to the sum of the RBF kernel. This implies that optimizing our repulsive term can also optimize the repulsive force, i.e., 
    \begin{equation}
        \nabla_{Q} \cR\big(Q_{w_j}(s, a^{\rm ood}) \propto \frac{\sum_{j=1}^{n} \nabla_{{Q}^{(i)}_{t}} k({Q}^{(i)}_{t},{Q}^{(j)}_{t})}{\sum_{j=1}^{n}k({Q}^{(i)}_{t},{Q}^{(j)}_{t})}.
    \end{equation}
    Therefore, optimizing such loss also follows the update rule in Eq. (\ref{eq:discretized}) and can obtain the true Bayesian posterior.
\end{proof}

\section{Implementation Detail}\label{implementation}

\subsection{Experimental setups and baselines.}
All the experiments are conducted on D4RL with the `v2' datasets and 5 random seeds. We run EDAC and DRVF for 3M gradient steps and report the normalized scores of each policy. In Gym Mujoco tasks, the policy is evaluated every 1000 steps, and the evaluation contains 1 episode. In AntMaze tasks, the policy is evaluated every 20000 steps, and the return averages over 100 episodes.

For baselines, we cite the reported results of TD3-BC, PBRL, and SAC-N from their own publications \cite{td3bc-2021, PBRL-2022, EDAC-2021}. For IQL\footnote{https://github.com/ikostrikov/implicit\_q\_learning} and EDAC\footnote{https://github.com/snu-mllab/EDAC}, we reproduce the results with their official implementations. For ATAC\footnote{https://github.com/microsoft/ATAC}, we run the official codes for the results of the `expert' datasets and cite other results from the original paper \cite{ATAC-2022}. For CQL and MOPO, we cite the reported scores in \cite{PBRL-2022}. Since the reported results of UWAC rely on the `v0' datasets, we choose the reimplementation results on the `v2' datasets from PBRL to evaluate these methods in the same setting. We cite the results of COMBO from the original paper.

In Table 2, we adopt the reported performance of several baselines from \cite{rvs}. For IQL, we run its official implementations and report its performance on `v2' dataset. Since the official codes of EDAC or SAC-N cannot work in AntMaze tasks, we do not include them in comparison. Inspired by \cite{why-so-pess}, DRVF utilizes independent targets in AntMaze tasks and achieves competitive performance with a few networks. Considering the sparse rewards in Antaze tasks, we adopt the same way as prior works \cite{cql-2020, why-so-pess} and transform the rewards in the offline datasets by $4(r-0.5)$.

\begin{table}[ht]
\centering
\caption{Hyper-paramters of DRVF}
\footnotesize
\label{hyperparam}
\resizebox{0.6\linewidth}{!}{
\begin{tabular}{lccccc}
\toprule
 Environments & \makecell{number of \\ ensembles $M$} & $\eta_q$ & $\eta_{\rm ood}$ & $L_2$ decay & Layernorm \\
\midrule
halfcheetah-random & 2 & 50 & 1 & 0.01 & F \\
halfcheetah-medium & 3 & 50 & 1 & 0.05 & F \\
halfcheetah-medium-replay & 3 & 50 & 1 & 0.05 & F \\
halfcheetah-medium-expert & 5 & 50 & 5 & 0 & T \\ 
halfcheetah-expert & 5 & 10 & 1 & 0 & T \\
\midrule
walker2d-random & 5 & 50 & 5 & 0.5 & F \\
walker2d-medium & 5 & 5 & 3 & 0.003 & T \\ 
walker2d-medium-replay & 4 & 10 & 5 & 0.01 & T \\ 
walker2d-medium-expert & 5 & 20 & 5 & 0.001 & T \\ 
walker2d-expert & 3 & 10 & 5 & 0 & T \\
\midrule
hopper-random & 5 & 10 & 1 & 0 & F \\
hopper-medium & 5 & 5 & 3 & 0 & T \\ 
hopper-medium-replay & 5 & 50 & 3 & 0 & T \\
hopper-medium-expert & 5 & 1 & 3 & 0 & T \\
hopper-expert & 5 & 1 & 3 & 0 & T \\
\midrule
antmaze-umaze & 5 & 1 & 5 & 0 & T \\
antmaze-umaze-diverse & 10 & 1 & 10 & 0.001 & F \\
antmaze-medium-diverse & 10 & 1 & 10 & 0 & T \\
antmaze-medium-play & 10 & 1 & 10 & 0 & T \\
\bottomrule
\end{tabular}}
\end{table}
\textbf{Hyper-parameters.}~
Most hyper-parameters of DRVF are the same as those of EDAC. In our experiments, we use different loss weights for the Gym Mujoco tasks and AntMaze navigation tasks. In addition, we add weight decay and layer norm in some tasks to prevent overfitting or divergence. The hyper-parameter settings are shown in Table \ref{hyperparam}. `T/F' for Layernorm means whether using Layernorm layers. 

\section{More Experimental Results.}\label{statistic}

\begin{figure}[ht]
    \centering
    \includegraphics[width=0.8\linewidth]{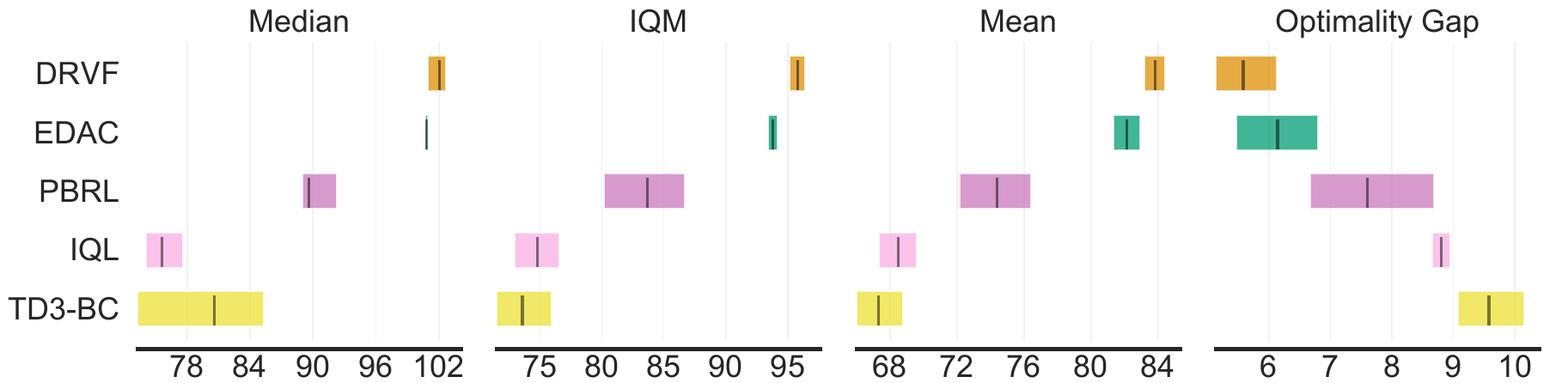}
    \caption{Aggregate metrics on D4RL with 95\% CIs based on Gym Mujoco tasks and 5 random seeds for each task. DRVF shows a higher median, mean, and IQM, and a lower optimality gap than other methods.}
    \label{statistical_1}
\end{figure}

We use more efficient and robust evaluation principles to characterize the statistical uncertainty. Specifically, we adopt the evaluation methods from \cite{agarwal2021deep} and present the comparisons in Fig. \ref{statistical_1}. We consider aggregate metrics, which include median, interquartile mean (IQM), optimality gap, and mean. IQM calculates the mean score of the middle 50\% runs with 95\% confidence intervals (CIs), and the optimality gap measures the number of runs where the algorithms fail to achieve a minimum score beyond which improvements are not very important. 
We compare our algorithm against TD3-BC, IQL, PBRL, and EDAC and conduct experiments on Gym Mujoco tasks with 5 random seeds. Higher mean, median, and IQM scores, and lower optimality gaps are expected.

\subsection{Results in AntMaze tasks.}

We also evaluate our method in challenging AntMaze navigation tasks, which require an 8-DoF ant robot to reach a target position in a maze scenario. In this task, the key challenges are dealing with sparse rewards and performing sub-trajectory stitching to reach the goal. We conduct experiments on two environments (umaze and medium) with two types of datasets. Table \ref{table2} summarizes the comparison of our method with behavior cloning (BC), Filtered (`Filt.') BC which learns from the best 10\% trajectories, and value-based methods, including CQL, TD3-BC, and IQL. We observe that our method surpasses the performance of BC and Filt. BC, as well as CQL and TD3-BC. Meanwhile, DRVF outperforms IQL in the umaze task with only 5 ensembles and obtains competitive performance in the other environments with 10 ensembles.

\begin{table}[!t]
\centering
\footnotesize
\caption{
\justifying
Performance of the baseline algorithms and the proposed method in AntMaze tasks. 
We do not report the performance of other baselines like PBRL and EDAC since their original implementations cannot obtain reasonable results in AntMaze domains.}
\resizebox{0.6\linewidth}{!}{
\begin{tabular}{lrrrrrr}
\toprule
 Environments & BC & Filt. BC & CQL & TD3-BC & IQL & DRVF \\ 
\midrule
 umaze & 54.6 & 60.0 & 23.4 & 78.6 & 88.6$\pm$0.5 & 94.2$\pm$3.1 \\
 umaze-diverse &  45.6 & 46.5 & 44.8 & 71.4 & 63.6$\pm$10.3 & 59.4$\pm$9.9 \\
 medium-diverse & 0.0 & 37.2 & 0.0 & 3.0 & 73.0$\pm$3.4 & 56.2$\pm$8.7 \\
 medium-play & 0.0 &42.1 & 0.0 & 10.6 & 74.0$\pm$3.6 & 65.2$\pm$14.4 \\
\midrule 
Average & 25.1 & 46.5 & 17.5 & 40.9 & 74.8$\pm$8.0 & 68.8$\pm$9.0 \\ 
\bottomrule
\end{tabular}}
\label{table2}
\end{table}

\section{More Ablation Studies}\label{app:ablation}

\subsection{The role of BNN}
DRVF obtains uncertainty estimation by sampling several $Q$-functions from the ensemble BNNs and uses this uncertainty to form the LCB of the value function. In this section, we discuss the necessity of Bayesian neural networks by comparing ensemble BNNs and other ensemble models without BNNs (e.g. PBRL or EDAC), whose main difference lies in the last layer of the critic network.

We first illustrate that other methods like PBRL, which does not use BNNs, cannot perform well with fewer ensembles, but may benefit from the repulsive term. We consider two settings: 1) original PBRL with fewer ensembles, 2) PBRL (fewer ensembles) combined with the repulsive term in our method. The number of ensembles used in PBRL is the same as ours, as shown in Fig. 4. We present the comparisons on 9 Gym-Mujoco tasks in Table \ref{pbrl}.
The second column illustrates that PBRL with fewer ensembles cannot achieve the same performance as DRVF. The third column demonstrates that PBRL may benefit from the OOD repulsive term in several tasks, but is still inferior to DRVF. This result indicates that the Bayesian inference module is necessary, and the repulsive regularization term may be useful for other ensemble-based methods.

\begin{table}[t]
\caption{
% \justifying 
Comparisons with PBRL with fewer ensembles and its combination with the repulsive term. Each method uses the same number of ensembles $M$, which is presented in brackets.}
\centering
\footnotesize
\resizebox{0.9\linewidth}{!}{
\begin{tabular}{lccc}
\toprule
Environments & PBRL with fewer ensembles & PBRL (fewer ensembles) +  repulsive regularization & DRVF\\
\midrule
halfcheetah-m         & 59.8  ($M$=3) & 60.3  ($M$=3)   & \textbf{69.2 $\pm$ 2.7} ($M$=3) \\
halfcheetah-m-r       & 46.7  ($M$=3)  & 52.3  ($M$=3)   & \textbf{66.7 $\pm$ 2.0} ($M$=3) \\
halfcheetah-m-e       & 93.3  ($M$=5)  & 93.6  ($M$=5)& \textbf{104.2 $\pm$ 2.4} ($M$=5)\\
walker2d-m   & 89.2  ($M$=5)  &91.1  ($M$=5)& \textbf{95.1 $\pm$ 2.7}  ($M$=5)\\
walker2d-m-r  & 90.9 ($M$=4) & 84.5  ($M$=4)  & \textbf{96.1 $\pm$ 2.3}  ($M$=4) \\
walker2d-m-e  & 109.6 ($M$=5) & 105.6 ($M$=5) & \textbf{112.4 $\pm$ 1.1}  ($M$=5)\\
hopper-m  & 25.2  ($M$=5) & 102.2  ($M$=5) & \textbf{102.0 $\pm$ 1.2}  ($M$=5)\\
hopper-m-r  & 98.7  ($M$=5) & 101.1 ($M$=5)  & \textbf{102.9 $\pm$ 1.1} ($M$=5)\\
hopper-m-e  & 110.3  ($M$=5) & 110.8  ($M$=5)& \textbf{112.3 $\pm$ 1.8}  ($M$=5) \\
\bottomrule
\end{tabular}}
\label{pbrl}
\end{table}

To further explore the importance of BNNs, we remove the repulsive term from DRVF and only use ensemble BNNs. We compare DRVF (only ensemble BNNs) with PBRL and EDAC, whose results are directly taken from their papers. We use the same number of ensembles as PBRL or EDAC in most tasks, and use more ensembles in several walker2d or hopper tasks. We remark that although more ensembles are required in several tasks, our method does not need pessimistic updates on OOD data in PBRL or gradient diversification in EDAC. The results in Table \ref{table:more-ensembles} show that the Bayesian posterior for uncertainty quantification, which is one of our main contributions, is sufficient for favorable performance. It also implies that the repulsive regularization term, which is another contribution of our method, can decrease the number of ensembles required and improve computational efficiency.

\begin{table}[t]
\caption{
% \justifying 
Comparisons of DRVF (only ensemble BNNs) with the results reported in PBRL and EDAC. The numbers of ensembles required $M$ are presented in brackets.}
\centering
\footnotesize
\resizebox{0.65\linewidth}{!}{
\begin{tabular}{lccc}
\toprule
Environments & PBRL-Prior & EDAC & DRVF (only ensemble BNNs)\\
\midrule
halfcheetah-m         & 58.2 $\pm$ 1.5 ($M$=10)  & 65.9$\pm$0.6 ($M$=10) & 68.1 ($M$=10) \\
halfcheetah-m-r       & 49.5$\pm$0.8 ($M$=10)  & 61.3$\pm$1.9 ($M$=10) & 62.8 ($M$=10)\\
halfcheetah-m-e       & 93.6$\pm$2.3 ($M$=10)  & 106.3$\pm$1.9 ($M$=10) & 105.6 ($M$=10)\\
walker2d-m   & 90.3$\pm$1.2 ($M$=10)  &92.5$\pm$0.8 ($M$=10) & 93.5 ($M$=20) \\
walker2d-m-r  & 86.2$\pm$3.4 ($M$=10)  & 87.0$\pm$2.3 ($M$=10) & 96.6 ($M$=10)  \\
walker2d-m-e  & 109.8$\pm$0.2 ($M$=10) & 114.7$\pm$0.9 ($M$=10) & 113.7 ($M$=10) \\
hopper-m  & 81.6$\pm$14.5 ($M$=10)  & 101.6$\pm$0.6 ($M$=50)  & 100.9 ($M$=50) \\
hopper-m-r  & 100.7$\pm$0.4 ($M$=10)  & 101.0$\pm$0.5 ($M$=50)  & 102.9 ($M$=10)\\
hopper-m-e  & 111.2$\pm$0.7 ($M$=10)  & 110.7$\pm$0.1 ($M$=50) & 111.4 ($M$=50)  \\
\bottomrule
\end{tabular}}
\label{table:more-ensembles}
\end{table}

\subsection{The role of the repulsive term}\label{ablation-repulsive}

Although Table \ref{pbrl} has shown that the repulsive term can improve the performance of PBRL in several tasks, in this section, we discuss the impact of the repulsive term in DRVF. The repulsive regularization term aims to maximize the diversity of $Q$ samples for OOD data. 

Firstly, we compare two settings: DRVF with the repulsive term and DRVF without the repulsive term. We note that in this experiment, a few ensembles ($M\leq 5$) are used, which is different from the comparison in Table \ref{table:more-ensembles}. Comparisons on 9 tasks are presented in Table \ref{ood-comparison}. We also present the normalized scores in evaluation in Fig. \ref{ablation-loss}. In the hopper-medium-replay, halfcheetah-medium, and halfcheetah-medium-replay tasks, optimizing ensemble BNNs without the repulsive regularization term is sufficient to obtain competitive performance. This result is consistent with Fig. \ref{ablation-loss}, where the repulsive term and the KL divergence have only a small impact on performance in the hopper-medium-replay task. In the other tasks, the repulsive regularization term is necessary for obtaining favorable performance. Therefore, we suggest that the repulsive term is critical for our method.

\begin{table}[ht]
\caption{Ablations on the repulsive term. The errors are evaluated across 5 random seeds.}
\centering
\footnotesize
\resizebox{0.65\linewidth}{!}{
\begin{tabular}{lcc}
\toprule
Environments & DRVF without the repulsive term & DRVF with the repulsive term \\
\midrule
halfcheetah-m         & 70.0 $\pm$ 1.6 & 69.2 $\pm$ 2.7 \\
halfcheetah-m-r       & 59.0 $\pm$ 13.5 & 66.7 $\pm$ 2.0 \\
halfcheetah-m-e       & 86.6 $\pm$ 41.9 & 104.2 $\pm$ 2.4 \\
walker2d-m            & 2.0 $\pm$ 3.4 & 95.1 $\pm$ 2.7  \\
walker2d-m-r          & 47.6 $\pm$ 48.2 & 96.1 $\pm$ 2.3 \\
walker2d-m-e          & 0.5 $\pm$ 1.0 & 112.4 $\pm$ 1.1 \\
hopper-m              & 1.7 $\pm$ 0.6 & 102.0 $\pm$ 1.2 \\
hopper-m-r            & 85.1 $\pm$ 34.4 & 102.9 $\pm$ 1.1 \\
hopper-m-e            & 5.0 $\pm$ 6.2 & 112.3 $\pm$ 1.8 \\   
\bottomrule
\end{tabular}}
\label{ood-comparison}
\end{table}

\begin{figure}[ht]
    \centering
\includegraphics[width=0.4\linewidth]{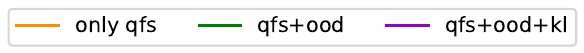}

\includegraphics[width=0.4\linewidth]{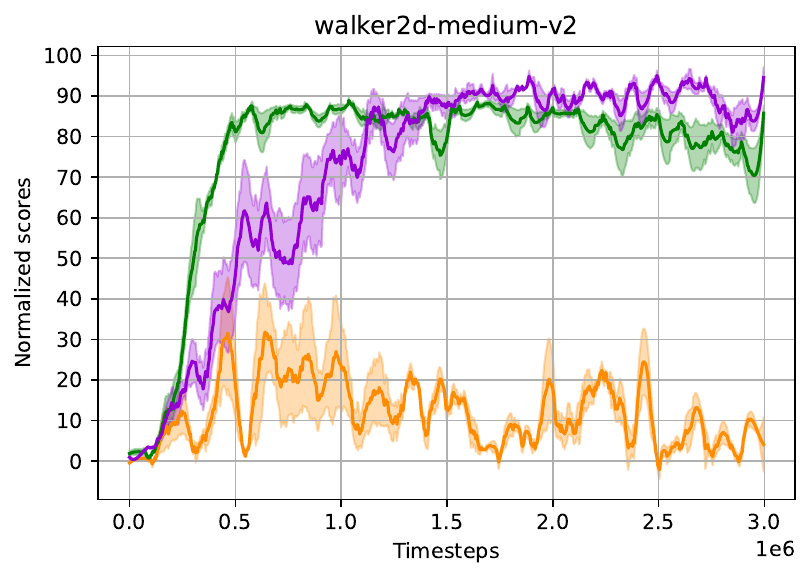}
\includegraphics[width=0.4\linewidth]{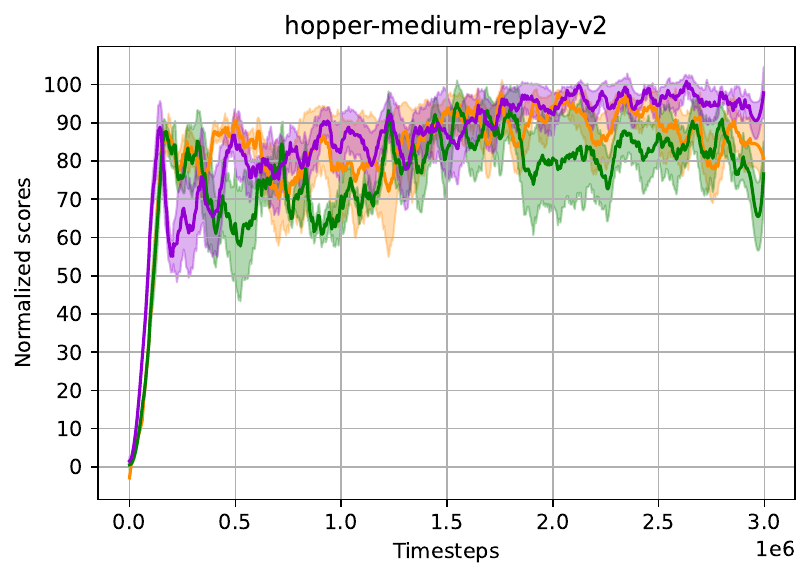}
    \caption{Ablation study on the loss functions components. `qfs', `ood', `kl' correspond to the TD error, the repulsive regulariztion term, and the KL divergence in Eq. (16). The shaded area indicates the variance. Only optimizing the TD error cannot learn effective policies in the walker2d-medium task, while showing inferior performance to optimizing the total loss in the hopper-medium-replay task.}
    \label{ablation-loss}
\end{figure}

Secondly, we investigate the impact of different ways to obtain OOD actions on the repulsive term and performance. We conduct experiments using two methods of generating OOD actions: randomly sampling from the action space or from the learned policy. The comparison is performed on 9 Gym-Mujoco tasks, and the normalized scores are shown in Table \ref{ood-actions-r}.
The results indicate that sampling from the current policy is reasonable and more effective, while random sampling OOD actions causes overestimation and final divergence of the value function in the halfcheetah and walker tasks.
Moreover, we remark that randomly sampling from the action space exhibits a significant discrepancy from sampling from the learned policies in most tasks, except for the hopper-medium-replay task where they show similar performance. We speculate that the repulsive term plays an important role in most tasks, while optimizing the loss function for ensemble BNNs alone is sufficient for competitive performance in the hopper-medium-replay task. This result is also consistent with Fig. \ref{ablation-loss}.

\begin{table}[ht]
\caption{Performance of different ways to obtain OOD actions.}
\footnotesize
\centering
\begin{tabular*}{0.65\linewidth}{lcc}
\toprule
Environments & Sampling from the  action space & Sampling from the learned policy \\
\midrule
halfcheetah-m         & 0.5$\pm$0.8 & 69.2 $\pm$ 2.7 \\
halfcheetah-m-r       & 0.3$\pm$0.8 & 66.7 $\pm$ 2.0 \\
halfcheetah-m-e       & 14.3$\pm$26.3 & 104.2 $\pm$ 2.4 \\
walker2d-m            & 12.2$\pm$11.1 & 95.1 $\pm$ 2.7  \\
walker2d-m-r          & 5.6$\pm$3.2 & 96.1 $\pm$ 2.3 \\
walker2d-m-e          & -0.4$\pm$0.2 & 112.4 $\pm$ 1.1 \\
hopper-m              & 12.9$\pm$15.3 & 102.0 $\pm$ 1.2 \\
hopper-m-r            & 80.9$\pm$32.7  & 102.9 $\pm$ 1.1 \\
hopper-m-e            & 8.6$\pm$6.8 & 112.3 $\pm$ 1.8 \\   
\bottomrule
\end{tabular*}
\label{ood-actions-r}
\end{table}

\subsection{The role of ensembles}

Ensembles are a key component of DRVF. We evaluate our algorithm with different numbers of $Q$-networks $M\in\{2, 3, 4, 5, 6\}$ and present the results in Fig. \ref{ablation-ensemble}. We observe that by using approximate Bayesian posterior and repulsive regularization, DRVF can learn effective policies with $M\approx 5$. Fig. 4 presents the minimum number of ensembles required for other tasks, indicating that DRVF needs fewer ensembles than EDAC and improves its parametric efficiency. Another suggestion from Fig. \ref{ablation-ensemble} is that more ensembles can provide more stable performance, which is consistent with our intuition that more ensembles can generate more samples and better approximation.

\begin{figure}[!t]
    \centering
    \includegraphics[width=0.6\linewidth]{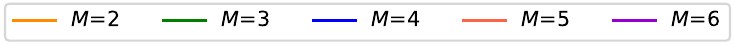}
    
    \includegraphics[width=0.4\linewidth]{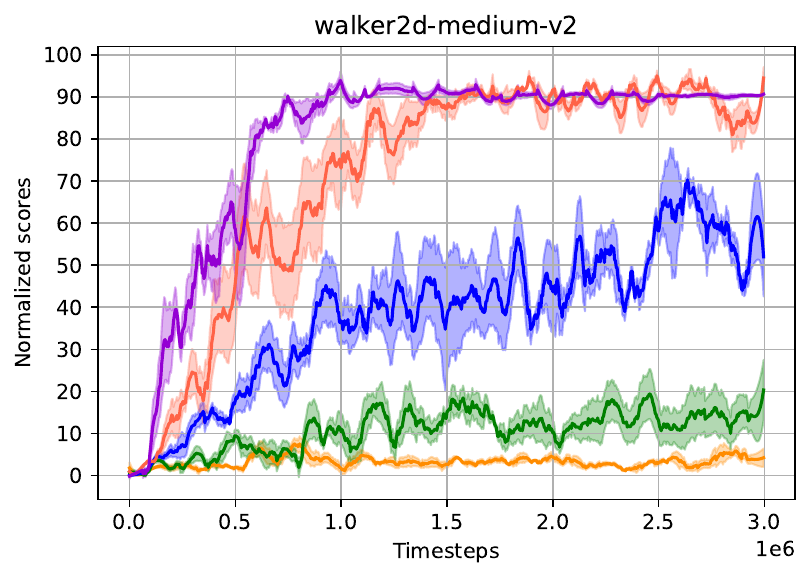}
    \includegraphics[width=0.4\linewidth]{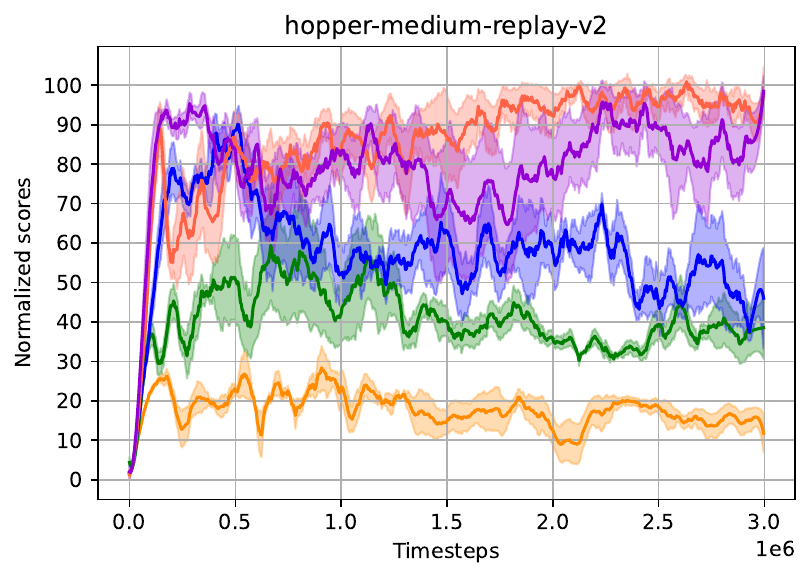}
    \caption{Ablation on the number of ensembles. The shaded area indicates the variance. DRVF can exhibit superior and stable performance with 5 or 6 ensembles.}
    \label{ablation-ensemble}
\end{figure}

\subsection{Bayesian Sampling}
Our method samples multiple $Q$-values from the Bayesian posterior and ensemble members to estimate the LCB of the $Q$-function. Intuitively, more samples can provide a better estimation of the LCB. We use $n$ to represent the number of samples from each ensemble member and compare the performance in two tasks, as displayed in Fig. \ref{ablation-bayesian-sample}. In general, we find that DRVF obtains reasonable performance even with a small number of samples (e.g., $n=2$). We speculate that variational methods can quickly obtain a posterior estimate. Meanwhile, the repulsive regularization diversifies the different samples, which drives the posterior samples widely distributed in the posterior and makes the LCB estimation easier.

\begin{figure}[ht]
    \centering
    \includegraphics[width=0.45\linewidth]{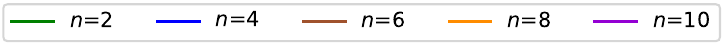}
    
    \includegraphics[width=0.4\linewidth]{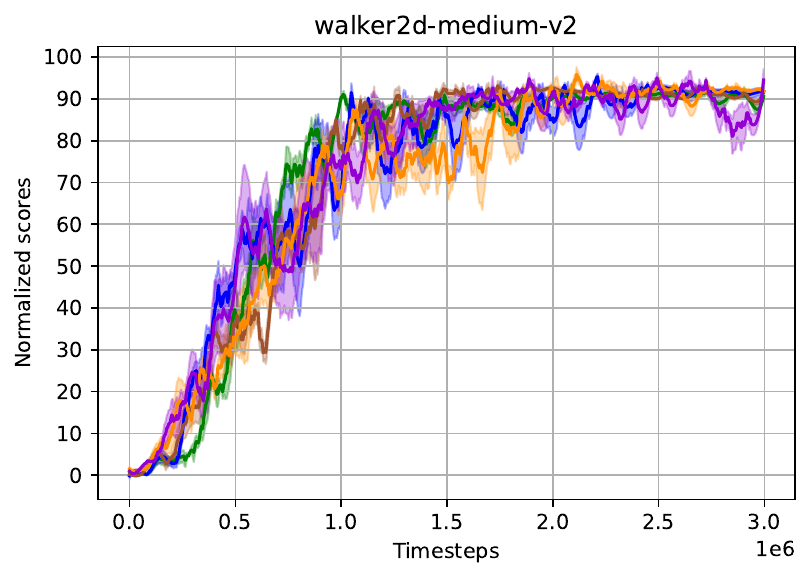}
    \includegraphics[width=0.4\linewidth]{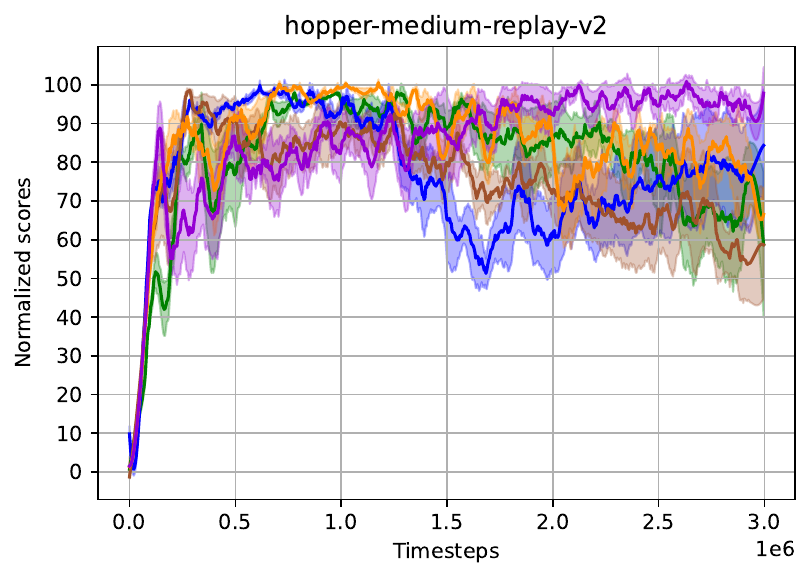}
    \caption{Ablation on the number of Bayesian samples. The shaded area indicates the variance. In the  walker2d-medium environment, a small number of samples are enough, while more samples are needed to obtain stable performance in the hopper-medium-replay environment.}
    \label{ablation-bayesian-sample}
\end{figure}

\subsection{Number of OOD Actions.}
The diversification of OOD data is implemented by OOD sampling and explicitly maximizing the uncertainty of $Q$-values for OOD actions. In the OOD sampling process, $K$ actions are sampled from the policy being trained. We perform experiments with $K \in \{1,3,5,7,10\}$ to analyze the sensitivity of our algorithm to the number of OOD actions. The curves in Fig. \ref{ablation-OOD} indicate that the number of OOD samples does not significantly influence policy training in the walker-medium task. Nevertheless, in hopper-medium-replay task, the value distribution requires more OOD actions to diversify the posterior samples.

\begin{figure}[ht]
    \centering
    \includegraphics[width=0.45\linewidth]{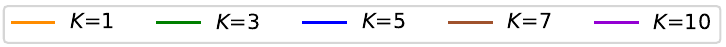}
    
    \includegraphics[width=0.4\linewidth]{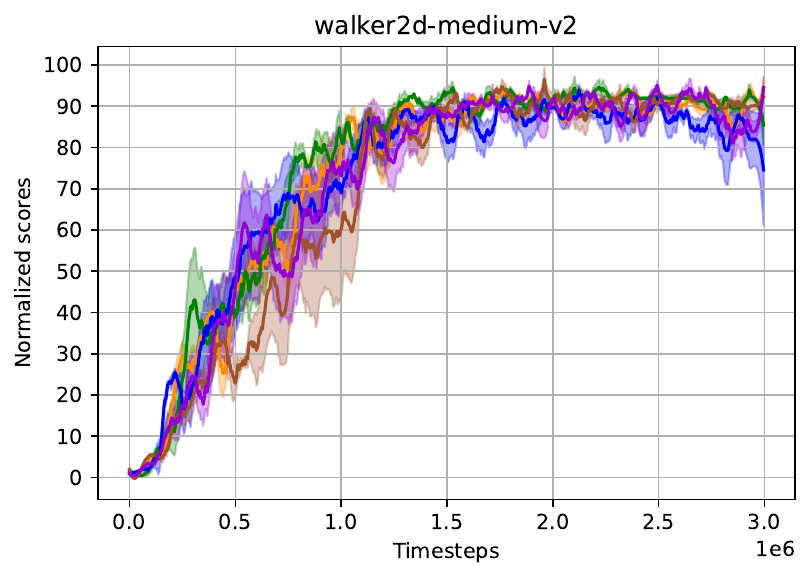}
    \includegraphics[width=0.4\linewidth]{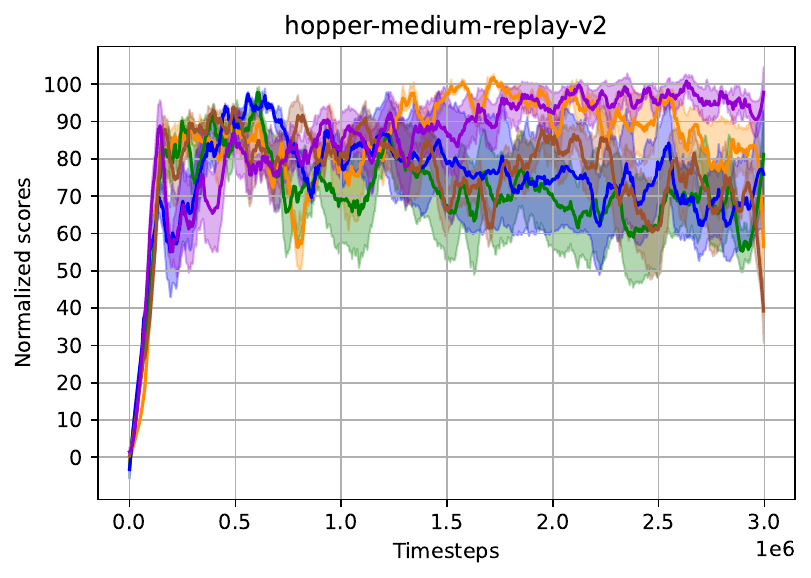}
    \caption{Ablation on the number of OOD actions. The shaded area indicates the variance. We tend to sample $K=10$ OOD actions in both environments to keep strong performance.}
    \label{ablation-OOD}
\end{figure}

\subsection{Comparison with dropout ensembles.}
We also compare DRVF with dropout ensembles and present the results in Fig. \ref{ablation-dropout}. We evaluate dropout ensembles with the same number of ensembles as DRVF and equip them with gradient diversification from EDAC or repulsive regularization from DRVF. However, we find that dropout ensembles cannot learn effective policies in all settings. We speculate that it is difficult for dropout ensembles to provide reliable uncertainty quantification and value estimation with a small number of ensembles \cite{UWAC-2020}.

\begin{figure}[ht]
    \centering
    \includegraphics[width=0.9\linewidth]{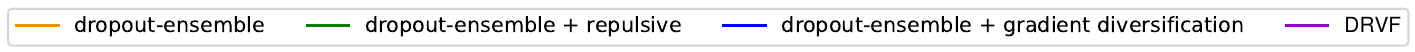}
    
    \includegraphics[width=0.4\linewidth]{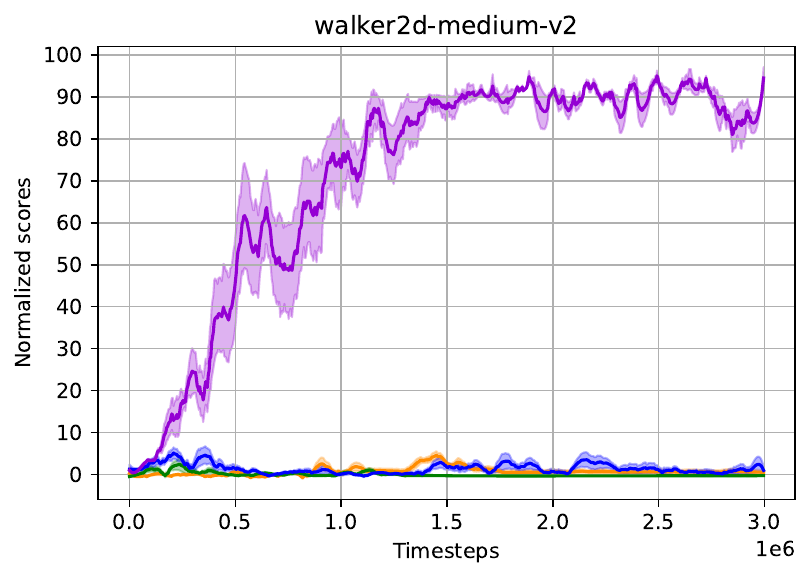}
    \includegraphics[width=0.4\linewidth]{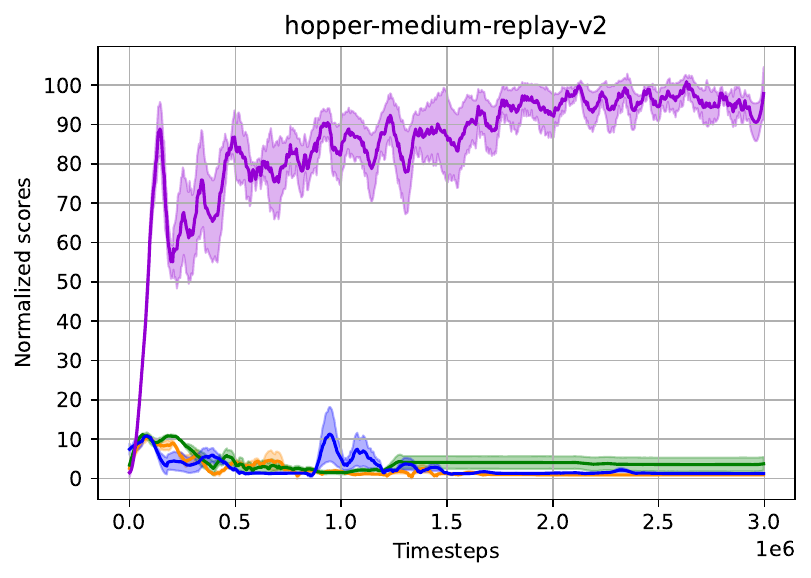}
    \caption{Comparisons with dropout ensembles. We evaluate dropout ensembles and DRVF with the same number of ensembles ($M=$ 5).}
    \label{ablation-dropout}
\end{figure}
\end{document}